\documentclass[11pt]{article}

\usepackage{dsfont}

\def\colorful{0}

\usepackage{amsthm,amsfonts,amsmath,amssymb,epsfig,color,float,graphicx,verbatim, enumitem}
\usepackage{multirow}
\usepackage[linesnumbered,lined,boxed,commentsnumbered,ruled,vlined]{algorithm2e}

\newif\ifhyper\IfFileExists{hyperref.sty}{\hypertrue}{\hyperfalse}
\hypertrue
\ifhyper\usepackage{hyperref}\fi

\usepackage{enumitem}

\usepackage{framed}
\usepackage{nicefrac}

\def\nnewcolor{1}
\ifnum\nnewcolor=1

\fi
\ifnum\nnewcolor=0

\fi

\ifnum\colorful=1
\newcommand{\new}[1]{{\color{red} #1}}

\else
\newcommand{\new}[1]{{#1}}

\fi

\newtheorem{theorem}{Theorem}[section]

\newtheorem{lemma}[theorem]{Lemma}
\newtheorem{informal theorem}[theorem]{Theorem (informal statement)}
\newtheorem{condition}[theorem]{Condition}
\newtheorem{proposition}[theorem]{Proposition}

\newtheorem{claim}[theorem]{Claim}
\newtheorem{fact}[theorem]{Fact}

\theoremstyle{definition}
\newtheorem{definition}[theorem]{Definition}
\newcommand{\eqdef}{\stackrel{{\mathrm {\footnotesize def}}}{=}}

\newcommand{\R}{\mathbb{R}}

\newcommand{\Z}{\mathbb{Z}}

\newcommand{\E}{\mathbf{E}}
\newcommand{\eps}{\epsilon}
\newcommand{\dtv}{d_{\mathrm TV}}
\newcommand{\dkl}{d_{\mathrm KL}}

\renewcommand{\Pr}{\mathbf{Pr}}
\newcommand{\poly}{\mathrm{poly}}
\newcommand{\var}{\mathbf{Var}}
\newcommand{\cov}{\mathbf{Cov}}

\newcommand{\Diam}{\mathrm{diam}}

\newcommand{\littlesum}{\mathop{\textstyle \sum}}

\newcommand{\wt}{\widetilde}
\newcommand{\wh}{\widehat}

\newcommand{\ltwo}[1]{\left\lVert#1\right\rVert_2}

\newcommand{\todo}[1]{\typeout{TODO: \the\inputlineno: #1}\textbf{[[[ #1 ]]]}}

\title{Outlier-Robust Learning of Ising Models\\ Under Dobrushin's Condition}

\author{
Ilias Diakonikolas\thanks{Supported by NSF Award CCF-1652862 (CAREER) and a Sloan Research Fellowship.}\\
University of Wisconsin Madison\\
{\tt ilias@cs.wisc.edu}\\
\and
Daniel M. Kane\thanks{Supported by NSF Award CCF-1553288 (CAREER) and a Sloan Research Fellowship.}\\
University of California, San Diego\\
{\tt dakane@cs.ucsd.edu}\\
\and
Alistair Stewart\thanks{Part of this research was performed while the author was a postdoctoral researcher at USC, supported
by Ilias Diakonikolas' startup funding.}\\
Web 3 Foundation\\
{\tt stewart.al@gmail.com}\\
\and
Yuxin Sun\thanks{Supported by NSF Award CCF-1652862 (CAREER).}\\
University of Wisconsin Madison\\
{\tt yxsun@cs.wisc.edu}\\
}

\begin{document}

\maketitle

\begin{abstract}
We study the problem of learning Ising models satisfying Dobrushin's
condition in the outlier-robust setting where a constant fraction of
the samples are adversarially corrupted. Our main result is to provide
the first computationally efficient robust learning algorithm for this
problem with near-optimal error guarantees. Our algorithm can be seen
as a special case of an algorithm for robustly learning a distribution
from a general exponential family. To prove its correctness for
Ising models, we establish new anti-concentration results 
for degree-$2$ polynomials of Ising models 
that may be of independent interest.
\end{abstract}

\setcounter{page}{0}

\thispagestyle{empty}

\newpage

\section{Introduction} \label{sec:intro}

\subsection{Background and Motivation} \label{ssec:background}

Probabilistic graphical models~\cite{Koller:2009} provide a rich and
unifying framework to model structured high-dimensional distributions in terms of
the local dependencies between the input variables.
The problem of inference in graphical models arises in many applications
across scientific disciplines, see, e.g.,~\cite{wainwright2008graphical}.
In this work, we study the inverse problem of learning graphical models from data.
Various formalizations of this general learning problem have been studied
during the past five decades, see, e.g.,~\cite{Chow68, Dasgupta97, Abbeel:2006,
WainwrightRL06, AnandkumarHHK12, SanthanamW12, LohW12, BreslerMS13,
BreslerGS14a, Bresler15, KlivansM17}, resulting in general theory
and algorithms for various settings.

In this work, we focus on learning Ising models~\cite{Ising25},
the prototypical family of binary undirected graphical models
with applications in computer vision, computational biology,
and statistical physics~\cite{Li-book-MRFs-09, JEMF06, Fels09, chatterjee2005concentration}.

\begin{definition}[Ising Model] \label{def:ising}
Given a real symmetric matrix $(\theta_{ij})_{i,j\in[d]}$ with zero diagonal and a real vector $(\theta_i)_{i\in[d]}$,
the Ising model distribution $P_\theta$ is defined as follows: For any $x\in\{\pm1\}^d$,
$P_\theta(x)=\frac{1}{Z(\theta)}\exp \big((1/2) \sum_{i,j\in[d]}{\theta_{ij}x_ix_j}+\sum_{i=1}^{d}{\theta_ix_i}\big)$,
where the normalizing factor $Z(\theta)$ is called the partition function.
We call the matrix $(\theta_{ij})_{i,j\in[d]} \in \R^{d \times d}$ the interaction matrix and
the vector $(\theta_i)_{i\in[d]} \in \R^d$ the external field.
\end{definition}

\noindent The majority of prior algorithmic work on learning Ising models studies
the ``structure learning'' problem, i.e., the problem of learning the structure of the underlying
graph of non-zero entries of the interaction matrix, 
see, e.g.,~\cite{Bresler15, KlivansM17, HamiltonKM17}.
In this line of work, it is assumed that the true graph satisfies some structural property
(typically, a tree or bounded-degree structure) and certain (upper and lower) bounds are imposed
on the underlying parameters. Such assumptions are information-theoretically necessary for this
version of the problem. An emerging line of work studies the {\em distribution learning} problem,
i.e., the task of computing an Ising model that is close to the target in total variation distance,
see, e.g.,~\cite{dagan2020estimating, DP20, BGPV20} for a few recent papers.

Here we study the algorithmic problem of learning Ising models
in the presence of {\em adversarially corrupted data}.
We focus on the following standard data corruption model
that generalizes Huber's contamination model~\cite{Huber64}.


\begin{definition}[Total Variation Contamination] \label{def:adv}
Given $0< \epsilon < 1/2$ and a class of distributions $\mathcal{F}$ on $\mathbb{R}^d$,
the \emph{adversary} operates as follows: The algorithm specifies the number of samples $n$.
The adversary knows the true target distribution $X \in \mathcal{F}$ and selects a distribution $F$ such that
$\dtv(F, X) \leq \eps$. Then $n$ i.i.d.\ samples are drawn from $F$ and are given as input to the algorithm.
We say that a set of samples is {\em $\epsilon$-corrupted} if it is generated by this process.
\end{definition}


\noindent Intuitively, the parameter $\epsilon$ in Definition~\ref{def:adv}
quantifies the power of the adversary. The total variation contamination
model is strictly stronger than Huber's contamination model.
Recall that in Huber's model~\cite{Huber64}, the adversary generates
samples from a mixture distribution $F$ of the form $F = (1-\epsilon) X + \epsilon N$,
where $X$ is the unknown target distribution and $N$ is an adversarially chosen noise distribution.
That is, in Huber's model the adversary is only allowed to add outliers.

The contamination setting we consider is standard in robust statistics~\cite{HampelEtalBook86, Huber09},
a field which seeks to develop  {\em outlier-robust}  estimators  --- algorithms
that can tolerate a {\em constant fraction} of corrupted datapoints, independent of the dimension.
Classical work, starting with Tukey and Huber in the 1960s, developed statistically optimal
robust estimators for a number of settings. However, these early methods lead to
exponential-time algorithms, even for the most basic high-dimensional estimation tasks (e.g., mean estimation).

Two works from the theoretical CS community~\cite{DKKLMS16, LaiRV16} developed
the first efficient robust learning algorithms for  ``simple''
high-dimensional tasks, including mean and covariance estimation. 
Since these early works, we have witnessed substantial progress in algorithmic 
robust high-dimensional statistics by several communities. 
The ideas and techniques developed in~\cite{DKKLMS16} 
have been generalized to give efficient robust estimators for a range of models, 
including sparse models~\cite{BDLS17, DKKPS19-sparse}, 
mixture models~\cite{DKS18-list, HL18-sos, KSS18-sos, DHKK20, BK20, LM20, BD+20}, 
and general stochastic optimization~\cite{DKK+19-sever, PSBR18}. 
Intriguingly, some of these ideas have found applications 
in exploratory data analysis and adversarial machine learning, 
see, e.g.,~\cite{DKK+17, TranLM18, DKK+19-sever}. 
The reader is referred to~\cite{DK19-survey} for a  recent survey.

Prior algorithmic work on learning graphical models has almost exclusively
focused on the uncontaminated setting, where the data are i.i.d.\ samples from the distribution of interest.
We remark that recent work~\cite{HamiltonKM17, GKK19-Ising, KSC20} has developed algorithms
for structure learning in the {\em independent failures model},
where the coordinates of each example are independently flipped/missing with some probability.
On the other hand,~\cite{Lindgren19} point out that structure learning 
becomes information-theoretically impossible in the contamination model we consider here, 
if an adversary is allowed to corrupt even a tiny fraction of the samples.

The most relevant algorithmic work we are aware of in the contamination
model is~\cite{CDKS18-bn}, which developed
an outlier-robust learner for low-degree Bayes nets (directed graphical models)
with known graph structure. We also note that very recent work~\cite{PrasadSBR20}
developed nearly tight {\em sample complexity} bounds for learning Ising models
in Huber's contamination model under various structural assumptions ---
albeit by using underlying estimators that run in exponential time.

\subsection{Our Contributions} \label{ssec:results}

In this work, we study the following version of the learning problem:
{\em Given a multiset of corrupted samples from an unknown
Ising model, the goal is to learn the underlying distribution in total variation distance.}
This is a natural (and standard) formulation of distribution learning that has been studied
extensively in the literature (in both the i.i.d.\ regime and in the contaminated setting).
{\em Our main result is the first computationally efficient
outlier-robust estimator for Ising models} in this setting, under some natural assumptions.
We note that we do not make structural assumptions about the underlying graph --- our algorithms
work for Ising models on the complete graph.

To state our contributions in detail, we require some additional terminology.

\begin{definition}[Dobrushin's condition]\label{def:high-temp}
Given an Ising model $P_{\theta}$ with interaction matrix $(\theta_{ij})_{i,j\in[d]}$
and external field $(\theta_i)_{i\in[d]}$, we say that it satisfies Dobrushin's condition
if $\max_{i\in[d]}{\sum_{j\ne i}|\theta_{ij}|}\le1-\eta$, for some constant $0<\eta<1$.
\end{definition}

\noindent Dobrushin's condition for Ising models is a classical assumption
needed to rule out certain pathological behaviors.
This condition is standard in various areas, including statistical physics, machine learning,
and theoretical CS~\cite{kulske2003concentration, gotze2019higher, dagan2020estimating, adamczak2019note, gheissari2018concentration, marton2015logarithmic}.

Our main result is an efficient algorithm for outlier-robust learning of Ising
models with zero external field satisfying Dobrushin's condition.

\begin{theorem}[Robustly Learning Ising Models Without External Field]\label{thm:main-zero-ext}
Let $X\sim P_{\theta^*}$ be an Ising model without external field satisfying Dobrushin's condition for some
\new{universal constant} $\eta>0$. There is a universal constant $\epsilon_0>0$ such that the following holds:
Let $0< \epsilon < \epsilon_0$ and $S'$ be an $\epsilon$-corrupted set of $N$ samples from $P_{\theta^*}$.
There is a $\poly(N, d)$ time algorithm that,
for some $N=\wt{O}_\eta(d^2/\epsilon^2)$, on input $S'$ and $\epsilon$,
returns a symmetric matrix $\wh{\theta} \in \R^{d\times d}$
such that with probability at least $99/100$, we have that
$\|\wh{\theta}-\theta^*\|_F\le O_\eta(\epsilon\log(1/\epsilon))$.
Moreover, the Ising model distribution $P_{\wh{\theta}}$ satisfies Dobrushin's condition
and $\dtv(P_{\wh{\theta}},P_{\theta^*})\le O_\eta(\|\wh{\theta}-\theta^*\|_F)
\le O_\eta(\epsilon\log(1/\epsilon))$.
\end{theorem}

Some comments are in order.
We note that any robust estimator with contamination parameter $\eps$
information-theoretically requires error $\Omega(\eps)$.
(This lower bound is standard and applies even for binary product distributions.)
That is, the error guarantee of our algorithm is optimal,
within logarithmic factors. Moreover, our algorithm is proper (i.e., it outputs
an Ising model) and performs parameter learning,
i.e., it estimates the unknown parameters of the model within sufficient accuracy
to yield the desired total variation distance guarantee.


Our techniques extend to yield an outlier-robust learning algorithm with similar 
error guarantee for Ising models with non-zero external field (under additional assumptions). 
An informal version of our algorithmic result for the non-zero external field case follows.

\begin{theorem}[Robustly Learning Ising Models with Non-Zero External Field, Informal Version] 
\label{thm:thm:main-nonzero-inf}
Let $\eps > 0$ be less than a sufficiently small constant. 
For any sufficiently small $\alpha\ge0$, there exists an $M\ge0$ such that 
if $P_{\theta^*}$ is an Ising model in $d$ dimensions 
with $\max_i |\theta^*_i| \leq \alpha$ and $\max_i \sum_j |\theta^*_{ij}| \le M$, then 
there is some $N=\wt{O}_{\alpha}(d^2/\eps^2)$ and a  $\poly(N, d)$ time algorithm that 
given $\eps>0$ and a set of $N$ $\eps$-corrupted samples from $P_{\theta^*}$, 
it computes a $\wh{\theta}$ such that with probability at least $99/100$ 
it holds $\dtv(P_{\wh{\theta}},P_{\theta^*}) = O_{\alpha}(\eps \log(1/\eps))$.
\end{theorem}

\noindent See Theorem~\ref{thm:ising-ext} for a more detailed formal statement.
For the non-zero external field case, the value $\wh{\theta}$ that we recover unfortunately
is not guaranteed to be close to $\theta^*$ in Frobenius norm.
In fact, this is the wrong norm to compare them in and
such an approximation is information-theoretically impossible.
However, we do still guarantee that the corresponding Ising model distribution $P_{\wh{\theta}}$
satisfies Dobrushin's condition and $\dtv(P_{\wh{\theta}},P_{\theta^*}) = O(\epsilon\log(1/\epsilon))$. 

To achieve both of the above results, we view the Ising model
as an instance of a general exponential family.

\begin{definition}[Exponential Family]\label{def:exp-family}
An exponential family in canonical form is a family of distributions $P_{\theta}$
supported on a set $\new{\mathcal{X}}$, where the parameter $\theta$ belongs to
some convex set $\Omega \subseteq \R^d$, with density function
$P_{\theta}(x) = \exp\left(\langle T(x),\theta\rangle- A(\theta)\right),\forall x\in\mathcal{X}$,
where $A(\theta)$ is the normalizing factor called log-partition function
and the vector $T(x)$ is called the sufficient statistics of $P_{\theta}$.
\end{definition}

As one of our main contributions, we provide a computationally efficient outlier-robust parameter learning algorithm
for exponential families  under the following condition.

\begin{condition}\label{exp-family-cond}
For an arbitrary $\theta\in\Omega$, the exponential family $P_{\theta}$ satisfies the following:
\begin{enumerate}
\item $\cov_{X\sim P_\theta}[T(X)] \succeq c_1\, I$,
where $c_1>0$ is a universal constant independent of $\theta$ and the dimension $d$ of $T(X)$.

\item $T(x)$ has sub-exponential tails for a universal constant $c_2>0$,
i.e., for any unit vector $v\in\mathbb{R}^d$, it holds that 
$\Pr_{X\sim P_{\theta}}[|\langle v,T(X)-\E[T(X)]\rangle|> t]\le2\exp(-c_2t)$, \new{for all $t>0$,}
where $c_2>0$ is a universal constant independent of $\theta$ and the dimension $d$ of $T(X)$.
\item \new{There is an algorithm that, given as input $\theta \in \Omega$ and $\gamma>0$, it runs
in $\poly(d,1/\gamma)$ time and it outputs i.i.d.\ samples from a distribution $D_{\gamma}$ such that
$\dtv(D_{\gamma}, P_{\theta}) \le\gamma$.}
\end{enumerate}
In addition, the diameter of $\Omega$ is bounded from above, and we can efficiently compute
approximate projections on $\Omega$. Specifically, it holds that $\Diam(\Omega) \leq \exp(\poly(d))$,
and for any $\delta>0$ and $z\in \R^d$, there is a $\poly(d,1/\delta)$ time algorithm
that computes a point $y\in\Omega$ such that $\ltwo{y-P_\Omega(z)}\le\delta$,
\new{where $P_{\Omega}$ is the projection operation}.
\end{condition}

\noindent For exponential families, we show:

\begin{theorem}[Robust Learning of Exponential Families]\label{thm:main-exp-family}
Let $P_{\theta^*}$ be an exponential family over $\mathcal{X}$ with sufficient statistics $T(x)$, 
where the parameter $\theta^*\in\Omega$ and $\Omega \subseteq \R^d$ is convex. 
Assume that \new{Condition~\ref{exp-family-cond} holds}.
\new{Let $0<\epsilon<\epsilon_0$, for some universal constant $\epsilon_0$,
and $S'$ be an $\epsilon$-corrupted set of $N$ samples from $P_{\theta^*}$.
There is a $\poly(N, d)$ time algorithm that, for some $N=\wt{O}(d/\epsilon^2)$, on input $S'$ and $\eps>0$, }
returns a vector $\wh{\theta} \in \Omega$ such that with probability at least $99/100$
we have that $\|\wh{\theta}-\theta^*\|_2\le O(\epsilon\log(1/\epsilon))$.
\new{In addition, $\dtv(P_{\wh{\theta}},P_{\theta^*})\le O(\|\wh{\theta}-\theta^*\|_2)\le O(\epsilon\log(1/\epsilon))$.}
\end{theorem}

As we will explain in the next subsection, our robust learning algorithm for Ising models 
(in both the zero and non-zero external field cases) is the algorithm given in Theorem~\ref{thm:main-exp-family}. 
The main technical challenge is in establishing correctness, i.e., showing that an Ising model under Dobrushin's condition
satisfies Condition~\ref{exp-family-cond}. \new{To achieve this, we develop new anti-concentration results
for degree-$2$ polynomial of Ising models that we believe may be of independent interest 
(see Theorems~\ref{thm:ac} and~\ref{thm:ac-2}).}

\subsection{Overview of Techniques} \label{ssec:techniques}

\new{Our outlier-robust learning algorithm for Ising models
is a special case of a robust learning algorithm for the
class of exponential families (satisfying Condition~\ref{exp-family-cond}).
We start with an intuitive description of this algorithm
followed by a brief sketch of the tools required to prove its correctness.}


To robustly learn a family of distributions in total variation distance,
one typically requires a set of relevant parameters and a ``parameter distance'',
so that sufficiently accurate approximation in parameter distance implies
approximation in total variation distance. For exponential families,
a natural set of parameters present themselves:
the expectation of the sufficient statistics of the distribution. 
Our strategy will be to robustly estimate this expectation.


Unfortunately, there is a wrinkle in this strategy which relates to
the scale in which we are working. On the one hand, in order to
robustly estimate the mean of a distribution, one needs to know some
sort of tail bounds on the set of clean samples;
and for these tail bounds to hold, we need to know the scale
at which we expect this decay to happen. On the other hand,
once we learn an approximation to the true mean of the sufficient statistics,
we need to relate the sizes of these errors to the errors
we will obtain in the underlying parameters for the family,
and to the total variation distance of the final distribution that we learn.
These relationships define certain natural scales for our problem,
and it is not clear how to obtain a robust algorithm
if these scales disagree (in such a case, the accuracy
to which we can learn the expectation of the sufficient statistics
might differ from the accuracy to which we need to learn it to obtain
good error in total variation distance) or if the relevant scale
depends on the underlying (unknown) parameters.

To resolve this issue, we need to make an assumption (Condition~\ref{exp-family-cond}).
Specifically, we need to assume that there is a convex set $\Omega$ of parameters
in our exponential family, such that any elements of the family inside this set
have sufficient statistics whose covariances are within constant multiples of each
other. This implies that the relevant scales for our problem are all
comparable.

From this point, there is a relatively straightforward algorithm that achieves suboptimal error.
After a change of variables, we can assume that within $\Omega$
all of the sufficient statistics have covariance proportional to the identity.
This allows us to use standard robust mean estimation algorithms (Fact~\ref{bounded-cov})
to estimate the mean of the sufficient statistics to error $O(\sqrt{\eps})$ in $\ell_2$-norm.
This in turn allows us to estimate our distribution to error
$O(\sqrt{\eps})$ in total variation distance.

To improve on this error guarantee, we will need to obtain better error in our
robust mean estimation algorithm. This can be achieved under the following assumptions:
(1) The distribution in question satisfies strong tail bounds.
(2) We know an accurate approximation to the covariance matrix of the distribution.
As for (1), it follows for general exponential families that their
sufficient statistics will have exponential tail bounds, which is
sufficient for us. For (2), we will need to already have a good
approximation of the underlying parameters of our distribution. This
gives rise to an iterative algorithm. If we know the underlying
parameters of our exponential family to error $\delta$, we can learn the mean of the
sufficient statistics --- and thus new approximations to the parameters ---
to error $O(\eps \log(1/\eps)+\sqrt{\delta \eps})$ (Lemma~\ref{lem:iterative-refinement}).
Iterating this several times, we can eventually achieve the near-optimal error of
$O(\eps\log(1/\eps))$.

Our result for Ising models is obtained via an application of the above algorithm.
Note that Ising models are a special case of an exponential
family, where the sufficient statistics are given by degree-$2$
polynomials. For the above algorithm to provably work, we need
to show that (under some reasonable conditions on parameters) the
covariance of the sufficient statistics is well-behaved. In
particular, we show that if the underlying parameters satisfy
the Dobrushin condition, the covariance matrix of the sufficient statistics
will be proportional to the identity (see Theorems~\ref{thm:ac} and~\ref{thm:ac-2}).

Interestingly,~\cite{dagan2020estimating} recently showed that this holds for the
covariance of the space of {\em degree-$1$} polynomials of such Ising models.
We need to generalize this to show that \new{$\var[X^TAX]$} is proportional to
$\|A\|_F^2$ for any symmetric matrix $A$ \new{with zero diagonal}.
To achieve this, we use a decoupling trick to reduce the problem to the degree-$1$ case.
We relate the variance of $X^TAX$ to $\E[|(X+Y)^TA(X-Y)|^2]$, for $X$ and $Y$ independent
copies of our distribution. If we condition on the set $S$ of coordinates where $X_i = Y_i$,
then $(X+Y)$ and $(X-Y)$ become independent Ising models.
By estimating the covariances of these
{\em linear} functions of these statistics, we can get a handle on the final bound.

\subsection{Organization}
After some technical preliminaries (Section~\ref{sec:prelims}),
in Section~\ref{sec:exp} we prove Theorem~\ref{thm:main-exp-family}.
In Section~\ref{sec:ising} we establish Theorem~\ref{thm:main-zero-ext}, 
and in Section~\ref{sec:ising-external} we establish Theorem~\ref{thm:thm:main-nonzero-inf}.

\section{Preliminaries} \label{sec:prelims}

\paragraph{Notation} For $d \in \Z_+$, we use $[d]$ to denote the set $\{1,\ldots,d\}$. 
Given a subset $S\subseteq [d]$, we will denote $-S=[d]\setminus S$. 
In particular, given $i\in[d]$, let $-i=[d]\setminus\{i\}$. 
Given a vector $a=(a_1,\ldots,a_d)$ and $S\subseteq[d]$, 
let $a_S$ denote the $|S|$-coordinate vector $\{a_i:i\in S\}$. 
Let $\mathbb{S}^{d-1} = \{x\in\mathbb{R}^d:\|x\|_2=1\}$ be the $d$-dimensional unit sphere. 
For $u,v\in\mathbb{R}^d$, we use $\langle u,v\rangle$ 
for the inner product of $u$ and $v$.

Given a real symmetric matrix $A\in\mathbb{R}^{d\times d}$, 
let $\|A\|_F \eqdef \sqrt{\sum_{i,j\in[d]}A_{ij}^2}$, 
let $\|A\|_2 \eqdef \max_{v\in\mathbb{S}^{d-1}}\|Av\|_2$, 
and let $\|A\|_{\infty} \eqdef \max_{i\in[d]}\sum_{j=1}^{d}|A_{ij}|$.
For symmetric matrices  $A,B\in\R^{d\times d}$, we say that $A\succeq B$ if $A-B$ 
is positive semi-definite (PSD), and $A\preceq B$ if $B-A$ is PSD. 

For two distributions $p,q$ over a probability space $\Omega$, 
let $\dtv(p,q)\eqdef\sup_{S\subseteq\Omega}|p(S)-q(S)|$ 
denote the total variation distance between $p$ and $q$ and
let $\dkl(p,q)\eqdef\int_{\Omega}\log\left(\frac{dp}{dq}\right)dp$ 
denote the KL-divergence of $p$ and $q$.

We use $\E[X],\var[X],\cov[X,Y]$ to denote the expectation of $X$, 
variance of $X$, and covariance of $X$ and $Y$ respectively.

We will use the following terminology.
\begin{definition} [Bounded Ising Model] \label{def:ising-bounded}
Given $M,\alpha>0$, we say that an Ising model distribution $P_\theta$ 
is $(M,\alpha)$-bounded if $\max_{i\in[d]}\sum_{j\ne i}|\theta_{ij}|\le M$ 
and $\max_{i\in[d]}|\theta_i|\le\alpha$.
\end{definition}

Intuitively, the first inequality states that the dependencies among 
the points are weak and the second inequality guarantees 
that the variance of each point is sufficiently large.

\paragraph{Sub-exponential Distributions} 
Here we present basic facts about sub-exponential distributions.
The reader is referred to~\cite{vershynin2018high}.

\begin{definition}[Sub-exponential Distribution]\label{def:exp-decay}
A distribution $D$ over $\R$ is sub-exponential
if there is a constant $c>0$ such that for any $t>0$, we have 
$\Pr_{X\sim D}\left[|X-\E[X]|>t\right] \le 2\exp(-c \, t)$. 
We say that a distribution $D'$ over $\R^d$ is sub-exponential if there is a constant $c'>0$ such that
for any unit vector $v\in\mathbb{S}^{d-1}$ and any $t>0$, we have that
$\Pr_{X\sim D'}\left[|\langle v,X-\E[X]\rangle|>t\right]\le2\exp(-c' \, t).$
\end{definition}

The following standard fact translates tail bounds to moment bounds.

\begin{fact}\label{fact:sub-exponential-property}
Let $X$ be a mean-zero random variable, and suppose that there is a constant $K>0$ such that for any $t>0$,
$\Pr[|X|>t]\le2\exp(-t/K)$. Then there is a constant $C>0$ such that for any real number $p\ge1$,
$\E[|X|^p]\le (CKp)^p$. In addition, there is a constant $C'>0$ such that for any $0<|\lambda|<1/(C'K)$, 
we have that $\E[\exp(\lambda X)]\le\exp(C'^2K^2\lambda^2)$.
\end{fact}

\noindent Additional facts about sub-exponential distributions can be found in 
Appendix~\ref{ssec:subexp}.

\paragraph{Exponential Families}

Here we record some basic facts about exponential families. 
The proofs of these results are standard and presented in Appendix~\ref{app:exp-basics}.

The first fact says that for an arbitrary exponential family, 
the mean of the sufficient statistics is exactly the gradient of the log-partition function, 
and the covariance of the sufficient statistics is exactly the Hessian of the log-partition function.

\begin{fact}[see, e.g.,~\cite{wainwright2008graphical}]\label{fact:exp-family-p1}
Let $X \sim P_{\theta}$ be an exponential family over $\mathcal{X}$ with sufficient statistics $T(x)$ 
and probability density function $P_{\theta}(x)=\exp\left(\langle T(x),\theta \rangle - A(\theta) \right)$, 
$\theta \in \R^d$. Let $\mu_T=\E[T(X)]$ and $\Sigma_T=\cov[T(X)]$. 
Then, we have that $\nabla_\theta A(\theta)=\mu_T$ and 
$\nabla^2_\theta A(\theta)=\frac{\partial\mu_T}{\partial\theta}=\Sigma_T$.
\end{fact}

The following fact connects the KL-divergence between two exponential families 
with their parameters in an explicit form.

\begin{fact}[see, e.g.,~\cite{wainwright2008graphical}]\label{fact:exp-family-p2}
Let $P_{\theta}, P_{\theta'}$ be exponential families with probability density functions
$P_\theta(x) =\exp\left(\langle T(x),\theta\rangle- A(\theta)\right)$ and 
$P_{\theta'}(x) =\exp\left(\langle T(x),\theta'\rangle- A(\theta')\right)$,
where the parameters $\theta, \theta' \in\R^d$. 
Let $\mu_T=\E_{X\sim P_{\theta}}[T(X)]$, $\mu'_T=\E_{X\sim P_{\theta'}}[T(X)]$, 
$\Sigma_T=\cov_{X\sim P_{\theta}}[T(X)]$, and $\Sigma'_T=\cov_{X\sim P_{\theta'}}[T(X)]$. 
Then, we have that
\begin{align*}
\dkl(P_\theta,P_{\theta'}) = \langle\theta-\theta',\mu_T\rangle-A(\theta)+A(\theta').
\end{align*}
Combining this with Fact~\ref{fact:exp-family-p1}, we obtain that 
$\nabla_{\theta'}\dkl(P_\theta,P_{\theta'})=\mu'_T-\mu_T$ and $\nabla^2_{\theta'}\dkl(P_\theta,P_{\theta'})=\Sigma'_T$.
\end{fact}

The following lemma shows that for any exponential family $P_{\theta^*}$, 
if the sufficient statistics $T(x)$ is sub-exponential, then a good estimate for the parameter $\theta^*$ 
yields a good estimate in total variation distance.

\begin{lemma}\label{lem:para-to-dis}
Let $P_{\theta^*}$ be an exponential family over $\mathcal{X}$ with parameter $\theta^*\in\mathbb{R}^d$ 
and sufficient statistics $T(x)$. Let $\wh{\theta}\in\mathbb{R}^d$ such that $\|\wh{\theta}-\theta^*\|_2\le\delta$,
for some sufficiently small constant $\delta>0$. If for any unit vector $v\in\mathbb{R}^d$, 
$\Pr_{X\sim P_{\theta^*}}[|\langle v,T(X)-\E[T(X)]\rangle|> t]\le2\exp(-ct)$, for all $t>0$, 
then $\dtv(P_{\wh{\theta}},P_{\theta^*})\le c'\|\wh{\theta}-\theta^*\|_2$, for some constant $c'>0$.
\end{lemma}

\noindent We defer the proof of Lemma~\ref{lem:para-to-dis} to Appendix~\ref{app:para-to-dis}.

\paragraph{Basic Properties of Ising Models}

Here we present some basic properties of Ising models, which will be used throughout this paper.
The proofs of these results are standard and presented in Appendix~\ref{app:ising-basics}.

Our first property states that if we arbitrarily fix the states of an arbitrary set of points, 
the conditional distribution of other points is still an Ising model.

\begin{fact}\label{fact:cond}
Let $X\sim P_{\theta}$ be an Ising model supported on $\{\pm1\}^d$ and $I\subseteq[d]$. 
For any fixed vector $x_{-I}\in\{\pm1\}^{-I}$, the conditional distribution of $X_{I}$ over $\{\pm1\}^I$
conditioning on $X_{-I}=x_{-I}$ is an Ising model with interaction matrix $\theta'_{ij}=\theta_{ij}$, for all $i,j\in I$, 
and external field $\theta'_i=\theta_i+\sum_{j\notin I}{\theta_{ij}x_j}$, for all $i\in I$.
\end{fact}

Our second property states that for an arbitrary $(M,\alpha)$-bounded Ising model, 
every point has sufficiently large variance.

\begin{fact}\label{fact:marginal}
Let $X\sim P_\theta$ be an $(M,\alpha)$-bounded Ising model supported on $\{\pm1\}^d$. 
Then, for every $i\in[d]$ and $x_i\in\{\pm1\}$, we have that
$\frac{\exp(-2(\alpha+M))}{1+\exp(-2(\alpha+M))}\le\Pr[X_i=x_i]\le\frac{\exp(2(\alpha+M))}{1+\exp(2(\alpha+M))}$.
Therefore, we also have that
$\var[X_i]=4 \, \Pr[X_i=1] \, \Pr[X_i=-1]\ge4\left(\frac{\exp(-2(\alpha+M))}{1+\exp(-2(\alpha+M))}\right)^2$.
\end{fact}

\paragraph{Glauber Dynamics} Glauber dynamics is the canonical Markov chain
 for sampling from undirected graphical models. 
The dynamics on the Ising model defines a reversible, ergodic Markov chain 
with stationary distribution corresponding to the Ising model. 
The Glauber dynamics for Ising models proceeds as follows:
\begin{enumerate}
\item Start at any initial state $X^{(0)}\in\{\pm1\}^d$.
\item Pick a point $i\in[d]$ uniformly at random and update $X_i^{(t)}$ as follows:
\begin{align*}
X_i^{(t+1)}=x\quad\text{w.p.}\quad\frac{\exp\left(\theta_ix+\sum_{j\ne i}\theta_{ij}X_j^{(t)}x\right)}{\exp\left(\theta_i+\sum_{j\ne i}\theta_{ij}X_j^{(t)}\right)+\exp\left(-\theta_i-\sum_{j\ne i}\theta_{ij}X_j^{(t)}\right)} \;.
\end{align*}
\end{enumerate}
The Glauber dynamics for an Ising model satisfying Dobrushin's condition is rapidly mixing, 
i.e., it converges fast to the underlying distribution $P_{\theta}$. 
\begin{fact}[see, e.g.,~\cite{levin2017markov}]\label{fact:Glauber-dynamics}
Let $P_{\theta}$ be an Ising model satisfying Dobrushin's condition and $\gamma>0$. 
Then, after $t=\Omega(d(\log d+\log(1/\gamma)))$ steps of Glauber dynamics, we have that
$\dtv\left(X^{(t)},P_\theta\right)\le\gamma$.
\end{fact}

\noindent Fact~\ref{fact:Glauber-dynamics} tells us that given the parameter $\theta$, we can 
efficiently generate approximate random samples from the Ising model distribution $P_{\theta}$, 
as long as it satisfies Dobrushin's condition.

\paragraph{Concentration and Anti-concentration of Ising models}

Several recent works have studied the concentration and anti-concentration 
of functions of Ising models~\cite{gheissari2018concentration,gotze2019higher, daskalakis2017concentration, adamczak2019note}. Here we record some results which will be used throughout this article.

The following fact states that for any $(1-\eta,\alpha)$-bounded Ising model, for some constants $\eta,\alpha>0$, 
the corresponding \new{Ising model} distribution is sub-Gaussian.

\begin{fact}[\cite{gotze2019higher}]\label{fact:sub-Gaussian}
Let $P_{\theta}$ be an Ising model satisfying Dobrushin's condition, 
and $\max_{i\in[d]}|\theta_i|\le\alpha$ for some constant $\alpha>0$. 
Then there is a constant $c(\alpha,\eta)>0$ such that for any $b\in \R^d$ and any $t>0$, we have that
$\Pr_{X\sim P_\theta} \left[ \left| b^TX - \E\left[b^TX\right] \right| > t \right] 
\le 2 \exp\left(-\frac{t^2}{c(\alpha,\eta)\|b\|_2^2}\right)$,
where $\eta>0$ is the constant in Definition~\ref{def:high-temp}.
\end{fact}

The following concentration property for quadratic forms of Ising models 
will be used to establish appropriate concentration inequalities.

\begin{fact}[\cite{gheissari2018concentration}]\label{fact:var-ub}
Let $X\sim P_{\theta}$ be an Ising model satisfying Dobrushin's condition. 
Let $A\in \R^{d\times d}$ be a symmetric matrix with zero diagonal and $b\in \R^d$. 
For any $x\in\{\pm1\}^d$, define $f(x)=(x-v)^T \, A \, (x-v)+b^T \, x$, where $v=\E[X]$. 
Then there is a constant $c(\eta)>0$ such that
\begin{align*}
\var[f(X)]\le c(\eta)(\|A\|_F^2+\|b\|_2^2) \;,
\end{align*}
where $\eta$ is the constant in Definition~\ref{def:high-temp}.
\end{fact}

We will require the following fact, which states that if the Ising model satisfies Dobrushin's condition, 
then changing the state of a single point will have small influence on other ones.

\begin{fact}[\cite{dagan2020estimating}]\label{fact:conditional-mean-distance}
Let $P_{\theta}$ be an Ising model satisfying Dobrushin's condition. 
Fix $i\in[d]$ and let $\mu_{-i}^1$ denote the conditional expectation over $x_{-i}$ conditioning on $x_i=1$, 
and $\mu_{-i}^{-1}$ denote the conditional expectation over $x_{-i}$ conditioning on $x_i=-1$. Then, we have that
$\left\|\mu_{-i}^1-\mu_{-i}^{-1}\right\|_1\le 2(1-\eta)/\eta$, and
$\sum_{j\ne i}\left|\cov[X_i,X_j]\right|\le (1-\eta)/\eta$,
where $\eta>0$ is the constant in Definition~\ref{def:high-temp}.
\end{fact}

We will also require the following anti-concentration result for linear forms on bounded Ising models:

\begin{fact}[\cite{dagan2020estimating}]\label{fact:linear-anti-concentration}
Let $X\sim P_{\theta}$ be an $(M,\alpha)$-bounded Ising model, where $M,\alpha>0$ are constants. 
Then there is a constant $c(M,\alpha)>0$ such that for any vector $b\in \R^d$, we have that
\begin{align*}
\var[b^TX]\ge c(M,\alpha)\|b\|_2^2 \;.
\end{align*}
\end{fact}

\noindent As a consequence of Fact~\ref{fact:linear-anti-concentration}, for any $(M,\alpha)$-bounded Ising model \new{$X$}, 
we have that $\cov[X]\succeq c(M,\alpha)\,I$.

\paragraph {Maximum Likelihood Estimation} 
Given a set of i.i.d.\ samples 
$S=\{x_1,\ldots,x_n\}\in\mathcal{X}^n$ drawn from an exponential family
$P_{\theta}$ with sufficient statistics $T(x)$ and unknown parameter $\theta\in\Omega$, 
the principle of maximum likelihood allows us to compute an estimate 
$\wh{\theta} \in \Omega$ by maximizing
the likelihood of \new{$S$}, i.e., 
$l(\theta,S)=\frac{1}{n}\sum_{i=1}^{n}\ln P_\theta(x_i)=\frac{1}{n}\sum_{i=1}^{n}\left(\langle T(x_i),\theta\rangle-A(\theta)\right)=\langle\theta,\wh{\mu}_T\rangle-A(\theta)$,
where $\wh{\mu}_T=\frac{1}{n}\sum_{i=1}^{n}T(x_i)$ is the empirical mean 
of the sufficient statistics $T(x)$ defined by the point set $S$. 
Define $L(\theta,\mu_T)=\langle\theta,\mu_T\rangle-A(\theta)$ and fix $\mu_T$ 
to be the empirical mean $\wh{\mu}_T$. The maximum likelihood estimator $\wh{\theta}$ is chosen 
to maximize the objective function $L(\theta,\new{\wh{\mu}_T})$ over $\theta \in \Omega$.

The following lemma states that under suitable conditions, 
if we obtain a good estimate of the mean $\mu_T$ 
of the sufficient statistics $T(x)$, the maximum likelihood estimator (MLE) 
will be a good approximation of the parameter $\theta$.
For completeness, we present the proof in Appendix~\ref{app:para-to-mean}.

\begin{lemma}\label{lem:parameter-to-mean}
Let $P_{\theta^*}$ be an exponential family such that $\theta^*$ 
lies in a convex set $\Omega \subseteq \R^d$. 
Let $\mu_T^*=\E_{X\sim P_{\theta^*}}\left[T(X)\right]$ and $\Sigma_T^*=\cov_{X\sim P_{\theta^*}}\left[T(X)\right]$. 
Let $\mu'_T$ be an approximation of $\mu^*_T$ such that $\|\mu'_T-\mu^*_T\|_2\le\delta$. 
Let $\theta' \in \arg\max_{\theta \in \Omega}{L(\theta,\mu'_T)}$, 
where $L(\theta,\mu'_T)=\langle\theta,\mu'_T\rangle-A(\theta)$. 
If there is a universal constant $c>0$ such that $\cov_{X\sim P_\theta}\left[T(X)\right] \succeq c \, I$, for all $\theta\in\Omega$, 
then $\|\theta'-\theta^*\|_2 \le 2 \delta/c$.
\end{lemma}

\section{Robust Parameter Learning of Exponential Families} \label{sec:exp}

In Section~\ref{ssec:exp-reduction}, we give an efficient algorithm (Lemma~\ref{lem:mean-to-parameter})
that reduces parameter estimation of exponential families 
to the task of estimating the mean of the sufficient statistics.
In Sections~\ref{ssec:exp-alg} and~\ref{ssec:exp-lemma-proofs}, we describe and analyze 
our computationally efficient robust parameter learning algorithm 
for exponential families satisfying Condition~\ref{exp-family-cond}.

\subsection{Learning via Estimating the Mean of Sufficient Statistics } \label{ssec:exp-reduction}

\new{
\begin{lemma}\label{lem:mean-to-parameter}
Let $P_{\theta^*}$ be an exponential family with sufficient statistics $T(x)$, 
where $\theta^*\in\Omega$ and $\Omega\subseteq \R^d$ is convex. Assume that \new{Condition~\ref{exp-family-cond} holds}.
Let $\mu_T^*=\E_{X\sim P_{\theta^*}}[T(X)]$ and $\mu_T'$ be an approximation of $\mu_T^*$ such that 
$\|\mu_T'-\mu_T^*\|_2\le\delta$, for some $0<\delta<1$ sufficiently small. Let $0<\zeta<1$.
Then there is a $\poly(d,1/\delta,1/\zeta)$ time algorithm that, given input $\mu_T',\delta$ and $\zeta$, 
returns a vector $\wh{\theta}\in\Omega$ such that with probability at least $1-\zeta$ 
we have that $\|\wh{\theta}-\theta^*\|_2\le O(\delta)$.
\end{lemma}
}

We give a proof sketch here; the details are in Appendix~\ref{app:mean-to-parameter}. 
Let $\theta'=\arg\max_{\theta\in\Omega}L(\theta,\mu'_T)$. 
By Lemma~\ref{lem:parameter-to-mean}, we know that $\|\theta'-\theta^*\|_2\le O(\|\mu'_T-\mu_T^*\|_2)\le O(\delta)$. Since given any $\theta\in\Omega$ we can efficiently sample from a distribution within 
small total variation distance of $P_{\theta}$, we can efficiently approximate the gradient 
$\nabla_\theta(-L(\theta,\mu'_T))=\E_{X\sim P_{\theta}}[T(X)]-\mu'_T$.  
In addition, by Condition~\ref{exp-family-cond}, we can show that there exist constants $L,m>0$ 
such that $-L(\theta,\mu'_T)$ is $L$-smooth and $m$-strongly convex. 
Then we can apply projected gradient descent to efficiently obtain an estimate 
$\wh{\theta}$ of $\theta'$ with $\|\wh{\theta}-\theta'\|_2\le O(\delta)$. Therefore, 
we get that $\|\wh{\theta}-\theta^*\|_2\le\|\wh{\theta}-\theta'\|_2+\|\theta'-\theta^*\|_2\le O(\delta)$.

\subsection{Robust Parameter Learning Algorithm} \label{ssec:exp-alg}

The pseudocode of our algorithm is given in Algorithm~\ref{main-algorithm}.
We make essential use of the following previously known algorithms for robust mean estimation
under bounded and approximately known covariance assumptions.

\begin{fact}[\cite{DKK+17, SCV18}]\label{bounded-cov}
Let $D$ be a distribution supported on $\mathbb{R}^d$ with unknown mean $\mu$ 
and unknown covariance $\Sigma$ such that $\Sigma\preceq\sigma^2I$, for some $\sigma>0$. 
Let $0<\epsilon<\epsilon_0$, for some universal constant $\epsilon_0$, 
and $\delta=O(\sqrt{\epsilon})$. 
Given an $\epsilon$-corrupted set of $N$ samples drawn from $D$, 
\new{for some $N=\widetilde{O}(d/\epsilon)$,}
there is a $\poly(N, d)$ time algorithm that outputs a vector $\wh{\mu}$ such that 
$\|\wh{\mu}-\mu\|_2\le O(\sigma\delta)=O(\sigma\sqrt{\epsilon})$ with high probability.
\end{fact}

\begin{fact}[see, e.g.,~\cite{CDGW19}]\label{approx-cov}
Let $D$ be a distribution on $\mathbb{R}^d$ with unknown mean $\mu$ 
and unknown covariance $\Sigma$. 
Let $0<\epsilon<\epsilon_0$, for some universal constant $\epsilon_0$, 
$\tau \le O(\sqrt{\epsilon})$, and $\delta=O(\sqrt{\tau\epsilon}+\epsilon\log(1/\epsilon))$. 
Suppose that $D$ has sub-exponential tails and $\Sigma$ satisfies $\ltwo{\Sigma-I}\le\tau$. 
Given an $\epsilon$-corrupted set of $N$ samples drawn from $D$, 
\new{for some $N=\widetilde{O}(d/\epsilon^2)$,}
there is a $\poly(N, d)$ time algorithm that outputs a vector $\wh{\mu}$ such that
 $\|\wh{\mu}-\mu\|_2\le O(\delta)$ with high probability.
\end{fact}

Algorithm~\ref{main-algorithm} starts by applying the robust mean estimation routine of Fact~\ref{bounded-cov} and Lemma~\ref{lem:mean-to-parameter} 
to obtain an initial estimate $\theta^{(0)}$ with $\ell_2$-error $O(\sqrt{\epsilon})$. 
Starting from this rough estimate, 
the algorithm applies an iterative refinement procedure (see Fact~\ref{approx-cov} and 
Lemma~\ref{lem:iterative-refinement}) for $T=O(\log\log(1/\epsilon))$ iterations 
to achieve near-optimal $\ell_2$-error of $O(\epsilon\log(1/\epsilon))$.

\begin{algorithm}
{\small
\SetKwInOut{Input}{Input}
\SetKwInOut{Output}{Output}
\Input{$0<\epsilon<\epsilon_0$, $\epsilon$-corrupted set of $N = \new{\wt{O}(d/\eps^2)}$ 
samples from exponential family $P_{\theta^*}$ satisfying Condition~\ref{exp-family-cond}, 
with $\mu_T^*=\E_{X\sim P_{\theta^*}}[T(X)]$ and $\Sigma_T^*=\cov_{X\sim P_{\theta^*}}[T(X)]$.
}
\Output{Parameter $\wh{\theta}\in\mathbb{R}^d$ such that 
$\|\wh{\theta}-\theta^*\|_2\le O(\epsilon\log(1/\epsilon))$ \new{with high probability}.}

Let $\delta=O(\sqrt{\epsilon})$.\\
Compute $\wh{\mu}_T^{(0)}$ with $\|\wh{\mu}_T^{(0)}-\mu^*_T \|_2\le\delta$ 
by applying the robust mean estimation algorithm of Fact~\ref{bounded-cov}.\\
Compute $\theta^{(0)}\in\Omega$ by applying projected gradient descent 
to the function $-L\big(\theta,\wh{\mu}_T^{(0)}\big)$.\\
Let $\tau_0=O(\delta)$ be an upper bound of $\|\theta^{(0)}-\theta^* \|_2$.\\
Let $K=O(\log\log(1/\epsilon))$.\\
\For{$k$ = $0$ to $K-1$}{
Let $n=\wt{O}(d^2/\tau_k^2)$. Generate $X^{(1)},\ldots,X^{(n)}$ i.i.d.\,random samples 
such that $\dtv(X^{(i)},P_{\theta^{(k)}})\le\wt{O}(\tau_k^2/d^2)$. \label{step:ena}\\
Let $\mu_T^{(k)}=\frac{1}{n}\sum_{i=1}^{n}X^{(i)}$ and 
$\Sigma_T^{(k)}=\frac{1}{n}\sum_{i=1}^{n}{\big(X^{(i)}-\mu_T^{(k)}\big)\big(X^{(i)}-\mu_T^{(k)}\big)^T}$. \label{step:dyo}\\
Let $\delta=O(\sqrt{\epsilon\tau_k}+\epsilon\log(1/\epsilon))$ \label{step:tria}.\\
Compute $\wh{\mu}$ with $\big\|\wh{\mu}-\big(\Sigma^{(k)}_T\big)^{-1/2}\mu_{T}^*\big\|_2\le\delta$ 
by applying the robust mean estimation algorithm of Fact~\ref{approx-cov}.\\
Compute $\theta^{(k+1)}\in\Omega$ by applying projected gradient descent 
to the function $-L\big(\theta,\big(\Sigma_T^{(k)}\big)^{1/2}\wh{\mu}\big)$.\\
Let $\tau_{k+1}=O(\delta)$ be an upper bound of $\|\theta^{(k+1)}-\theta^* \|_2$. \label{step:tessa}\\
}
\Return{$\theta^{(K)}$.}
}
\caption{Robust parameter estimation for exponential families}\label{main-algorithm}
\end{algorithm}

To prove correctness, we require Lemmas~\ref{lem:parameter-to-cov} 
and~\ref{lem:iterative-refinement} below.
Roughly speaking, in each refinement step, we first apply Lemma~\ref{lem:parameter-to-cov} 
to obtain a covariance estimate $\Sigma^{(k)}_T$ given the current parameter estimate $\theta^{(k)}$. 
Then, by Lemma~\ref{lem:iterative-refinement}, 
we are able to obtain a more accurate parameter estimate $\theta^{(k+1)}$.

Lemma~\ref{lem:parameter-to-cov} shows that given an estimate $\theta'$ 
of the true parameter $\theta^*$ of the exponential family satisfying 
Condition~\ref{exp-family-cond} with $\ltwo{\theta'-\theta^*}\le\delta$, 
we can efficiently compute an estimate $\wh{\Sigma}_T$ of the true covariance 
$\Sigma_T^*$ such that $\big\|\wh{\Sigma}_T-\Sigma_T^*\big\|_2\le O(\delta)$ with high probability.

\new{
\begin{lemma}\label{lem:parameter-to-cov}
Let $P_\theta^*$ be an exponential family with sufficient statistics $T(x)$, 
where $\theta^*\in \Omega$ and $\Omega\subseteq\mathbb{R}^d$ is convex. 
Assume that Condition~\ref{exp-family-cond} holds.
Let $\Sigma_T^*=\cov_{X\sim P_{\theta^*}}[T(X)]$. 
Let $\theta' \new{\in \Omega}$ be such that 
$\ltwo{\theta'-\theta^*}\le\delta$, for some \new{$\delta>0$}. Let $0<\zeta<1$.
There is a $\mathrm{poly}(d,1/\delta,1/\zeta)$ 
algorithm that, given as input $\theta',\delta$ and $\zeta$, 
returns a $d\times d$ PSD matrix $\wh{\Sigma}_T$ such that 
with probability at least $1-\zeta$, we have that 
$\big\|\wh{\Sigma}_T-\Sigma^*_T\big\|_2 \le O(\delta)$.
\end{lemma}
}

\noindent \new{The algorithm establishing Lemma~\ref{lem:parameter-to-cov} is very 
simple -- it corresponds to lines~\ref{step:ena} and~\ref{step:dyo} of Algorithm~\ref{main-algorithm}. 
Roughly speaking, we first generate i.i.d.\ random samples from a distribution $Q$ 
which is close to $P_{\theta'}$, and then let $\wh{\Sigma}_T$ 
be the empirical covariance of these samples.
}


Lemma~\ref{lem:iterative-refinement} shows that, given a fairly accurate estimate of the 
covariance $\Sigma^*_T$ of the sufficient statistics of an exponential family $P_{\theta^*}$
satisfying Condition~\ref{exp-family-cond}, we can efficiently obtain a more accurate estimate of 
the target parameter $\theta^*$.

\new{
\begin{lemma}[Iterative Refinement]\label{lem:iterative-refinement}
Let $0<\delta<\delta_0$ for some universal constant $\delta_0$ sufficiently small. Let $0<\zeta<1$.
Assume that Condition~\ref{exp-family-cond} holds. Let $S'$ be an $\epsilon$-corrupted set of $N$ 
samples from $P_{\theta^*}$. There is an algorithm that, 
for some $N=\wt{O}(d/\epsilon^2)$, given $S'$, $\delta$, $\zeta$, and $\Sigma^{(k)}_T$ 
with $\big\|\Sigma_T^{(k)}-\Sigma_T^*\big\|_2\le\delta$, it runs in $\poly(N, 1/\delta,1/\zeta)$-time 
and outputs $\theta^{(k+1)}\in\Omega$ such that with probability at least $1-\zeta$ it holds that
$\|\theta^{(k+1)}-\theta^*\|_2\le O(\sqrt{\epsilon\delta}+\epsilon\log(1/\epsilon))$.
\end{lemma}
}

\new{The algorithm establishing Lemma~\ref{lem:iterative-refinement} 
corresponds to lines~\ref{step:tria} to~\ref{step:tessa} of Algorithm~\ref{main-algorithm}. 
The main idea is the following:
Let $Y=(\Sigma^{(k)}_T)^{-1/2}T(X)$. 
We can show that the covariance of $Y$ is close to identity and $Y$ is sub-exponential, 
for some universal constant $c>0$. Thus, we can apply the robust mean estimation algorithm of 
Fact~\ref{approx-cov} to obtain an estimate $\wh{\mu}$ of the mean of $Y$.
In addition, we can show that $(\Sigma^{(k)}_T)^{1/2}\wh{\mu}$ 
is a good estimate of $\mu^*=\E_{X\sim P_{\theta^*}}[T(X)]$, and therefore we can apply 
Lemma~\ref{lem:mean-to-parameter} to get a new estimate $\theta^{(k+1)}$.
}

We give the proofs of Lemmas~\ref{lem:parameter-to-cov} and~\ref{lem:iterative-refinement} 
in Section~\ref{ssec:exp-lemma-proofs}.
Here we show how these lemmas can be used to 
prove Theorem~\ref{thm:main-exp-family}.

\smallskip

\begin{proof}[Proof of Theorem~\ref{thm:main-exp-family}]
Algorithm~\ref{main-algorithm} starts by applying the robust mean estimation algorithm for 
bounded covariance distributions (Fact~\ref{bounded-cov}) to obtain an estimate $\mu_T^{(0)}$ 
of the true mean $\mu_T^*=\E_{X\sim P_{\theta^*}}[T(X)]$ such that $\big\|\mu_T^{(0)}-\mu_T^*\big\|_2\le O(\sqrt{\epsilon})$. 
Then it applies the algorithm of Lemma~\ref{lem:mean-to-parameter} 
to obtain an initial estimate $\theta^{(0)}$ of the underlying parameter $\theta^*$ 
with $\|\theta^{(0)}-\theta^*\|_2\le O(\sqrt{\epsilon})$.

In each refinement step $k$, assume that we have a current estimate $\theta^{(k)}$ of the true parameter $\theta^*$ 
such that $\|\theta^{(k)}-\theta^*\|_2\le\tau_k$, for some $\tau_k>0$. Algorithm~\ref{main-algorithm} first 
applies the algorithm of Lemma~\ref{lem:parameter-to-cov} to compute an estimate $\Sigma^{(k)}_T$ 
of the true covariance $\Sigma_T^*$ with $\big\|\Sigma^{(k)}_T-\Sigma_T^*\big\|_2\le O(\|\theta^{(k)}-\theta^*\|_2)\le O(\tau_k)$. 
Then it applies the algorithm of Lemma~\ref{lem:iterative-refinement} to obtain a more accurate estimate 
$\theta^{(k+1)}$ of the true parameter $\theta^*$ such that $\|\theta^{(k+1)}-\theta^*\|_2\le\tau_{k+1}$, 
where $\tau_{k+1}=O(\sqrt{\epsilon\tau_k}+\epsilon\log(1/\epsilon))$. 
After $K=O(\log\log(1/\epsilon))$ iterations, we obtain an estimate $\wh{\theta}=\theta^{(K)}$ such that 
$\|\wh{\theta}-\theta^*\|_2\le O(\epsilon\log(1/\epsilon))$.

To bound the sample complexity and the failure probability, we take $\zeta=1/\log(1/\epsilon)$ in 
Lemmas~\ref{lem:mean-to-parameter},~\ref{lem:parameter-to-cov} and~\ref{lem:iterative-refinement}. 
Therefore, the total sample complexity of the algorithm is $N=\wt{O}(dK/\epsilon^2)=\wt{O}(d/\epsilon^2)$ 
and the total failure probability is at most $O\left(K/\log(1/\epsilon)\right)\le1/100$ by a union bound. 
By Lemma~\ref{lem:para-to-dis}, it follows that 
$\dtv(P_{\wh{\theta}},P_{\theta^*})\le O(\|\wh{\theta}-\theta^*\|_2)\le O(\epsilon\log(1/\epsilon))$.
Finally, it is easy to verify that the overall algorithm runs in polynomial time.
\end{proof}

\subsection{Implementing the Iterative Refinement Steps} \label{ssec:exp-lemma-proofs}

In this subsection, we prove Lemmas~\ref{lem:parameter-to-cov} and~\ref{lem:iterative-refinement}. 

We will require a couple of additional technical tools.
The following proposition connects the third derivative of the log-partition \new{function} $A(\theta)$ 
of the exponential family \new{$P_{\theta}$} 
with the third moment of the sufficient statistics $T(x)$.

\begin{proposition}\label{prop:exponential-family-p3}
Let $P_\theta$ be an exponential family with sufficient statistics $T(x)$ and density
$P_\theta(x)=\exp\left(\langle T(x),\theta\rangle- A(\theta)\right)$, $\theta\in \R^d$. 
Let $\mu_T=\E_{X\sim P_\theta}[T(X)]$ and $\Sigma_T=\cov_{X\sim P_\theta}[T(X)]$. 
Then, for any $i,j,k\in[d]$, we have that
$\frac{\partial(\Sigma_T)_{ij}}{\partial\theta_k}=\E_{X\sim P_\theta}\left[(T(X)-\mu_T)_i(T(X)-\mu_T)_j(T(X)-\mu_T)_k\right]$.
\end{proposition}

The proof is deferred to Appendix~\ref{app:prop:exponential-family-p3}.
Using Proposition~\ref{prop:exponential-family-p3}, 
we can bound the difference between the covariance matrices 
of the sufficient statistics of two exponential families \new{with sub-exponential tails}
in terms of the difference between their parameters.

\begin{lemma}\label{third-moment}
Let $\Omega\subseteq\mathbb{R}^d$ be a convex set. Assume that for any $\theta\in\Omega$, the exponential family $P_\theta$ with sufficient statistics $T(x)$ has sub-exponential tails for a universal constant $c>0$, 
i.e., for any $\theta\in\Omega$ and any unit vector $v\in\mathbb{R}^d$, 
$\Pr_{X\sim P_\theta}[|\langle v,T(X)-\E_{X\sim P_\theta}[T(X)]\rangle|> t]\le2\exp(-ct)$. 
Then there is a constant $c'>0$ such that for any $\theta^1,\theta^2\in\Omega$, we have that
$\|\Sigma_T({\theta^1})-\Sigma_T({\theta^2})\|_2\le c'\|\theta^1-\theta^2\|_2$, 
where for any $\theta\in\Omega$, $\Sigma_T(\theta)=\cov_{X\sim P_{\theta}}[T(X)]$.
\end{lemma}

\noindent The proof of the lemma is given in Appendix~\ref{app:third-moment}. 
Lemma~\ref{third-moment} is key ingredient in the proof of Lemma~\ref{lem:parameter-to-cov} below.

\medskip

We now have the necessary tools to prove 
Lemmas~\ref{lem:parameter-to-cov} and~\ref{lem:iterative-refinement}.

\begin{proof}[Proof of Lemma~\ref{lem:parameter-to-cov}]
Let $\Sigma'_T=\cov_{X\sim P_{\theta'}}[T(X)]$. 
From Lemma~\ref{third-moment}, it follows that 
$\ltwo{\Sigma'_T-\Sigma^*_T}\le O(\|\theta'-\theta^*\|_2)=O(\delta)$. 
In addition, given $\theta'\in\Omega$, we can efficiently sample from a distribution 
within total variation distance $\gamma=\frac{\delta^2\zeta^2}{2d^2(\log d+\log(12/\zeta))}$ 
from $P_{\theta'}$. 

Since $P_{\theta'}$ is sub-exponential and $\gamma$ is sufficiently 
small, by standard properties of sub-exponential distributions and the data processing inequality,
it follows that the empirical estimate $\wh{\Sigma}_T$ satisfies $\|\wh{\Sigma}_T-\Sigma'_T\|_2\le O(\delta)$ 
with probability at least $1-\zeta$. (Formally, this follows by picking $t=\log(12/\zeta)$ and 
$n=\frac{d^2(\log d+\log(12/\zeta))}{\delta^2\zeta}$ in Claim~\ref{approx-inference}.)
This implies that  
$\|\wh{\Sigma}_T-\Sigma^*_T\|_2\le\|\wh{\Sigma}_T-\Sigma'_T\|_2+\|\Sigma'_T-\Sigma^*_T\|_2\le O(\delta)$,
completing the proof.
\end{proof}

\begin{proof}[Proof of Lemma~\ref{lem:iterative-refinement}]
\new{Let $Y=\big(\Sigma^{(k)}_T\big)^{-1/2}T(X)$, 
$\mu_Y=\E_{X\sim P_{\theta^*}}[Y]=\big(\Sigma^{(k)}_T\big)^{-1/2}\mu_T^*$, and 
$\Sigma_Y=\cov_{X\sim P_{\theta^*}}[Y]=\big(\Sigma^{(k)}_T\big)^{-1/2}\Sigma^*_T\big(\Sigma^{(k)}_T\big)^{-1/2}$. 
From Condition~\ref{exp-family-cond} and Fact~\ref{fact:sub-exponential-property}, 
we know that $c\,I\preceq\Sigma^*_T\preceq c'\,I$, for some universal constants $c'\ge c>0$. 
Since $\big\|\Sigma^{(k)}_T-\Sigma^*_T\big\|_2\le\delta\le\delta_0$, 
we have that $(c-\delta_0)\,I\preceq\Sigma_T^{(k)}\preceq(c'+\delta_0)\,I$ 
and for any unit vector $v\in\mathbb{S}^{d-1}$, we have that}
\begin{align*}
&\quad \big|v^T(\Sigma^{(k)}_T)^{-1/2}(\Sigma^{(k)}_T-\Sigma^*_T)(\Sigma^{(k)}_T)^{-1/2}v\big| 
\le \big\| (\Sigma^{(k)}_T)-\Sigma^*_T \big\|_2 \; \big\|(\Sigma^{(k)}_T)^{-1/2}v\big\|_2^2\\
&\le \big\|(\Sigma^{(k)}_T)-\Sigma^*_T\big\|_2 \; \big\|(\Sigma^{(k)}_T)^{-1/2}\big\|_2^2
= \big\| (\Sigma^{(k)}_T)-\Sigma^*_T\big\|_2  \; \big\|(\Sigma^{(k)}_T)^{-1} \big\|_2\le O(\delta) \;,
\end{align*}
which implies that
\begin{align*}
1-O(\delta)&=v^T (\Sigma^{(k)}_T)^{-1/2}\Sigma^{(k)}_T (\Sigma^{(k)}_T)^{-1/2}v-O(\delta) 
\le v^T (\Sigma^{(k)}_T)^{-1/2}\Sigma^*_T (\Sigma^{(k)}_T)^{-1/2}v\\
&\le v^T (\Sigma^{(k)}_T)^{-1/2}\Sigma^{(k)}_T (\Sigma^{(k)}_T)^{-1/2} v +  O(\delta)
=1+O(\delta) \;.
\end{align*}
Therefore, we have that 
$\|\Sigma_Y-I\|_2=\big\|\big(\Sigma^{(k)}_T\big)^{-1/2}\Sigma^*_T\big(\Sigma^{(k)}_T\big)^{-1/2}-I\big\|_2 \le O(\delta)$. In addition, since $T(x)$ has sub-exponential tails by Condition~\ref{exp-family-cond} 
and $\Sigma_T^{(k)}\succeq(c-\delta_0)\,I$, we know that $Y$ also has sub-exponential tails. 
Therefore, we can apply the robust mean estimation algorithm 
for approximately known covariance distributions (Fact~\ref{approx-cov}) 
to obtain an estimate $\wh{\mu}$ of $\mu_Y$ such that 
$\|\wh{\mu}-\mu_Y\|_2\le O(\sqrt{\epsilon\delta}+\epsilon\log(1/\epsilon))$ 
with probability at least $1-\zeta/2$. 
We thus have that
\begin{align*}
\big\|\big(\Sigma^{(k)}_T\big)^{1/2}\wh{\mu}-\mu_T^*\big\|_2
&=\big\|\big(\Sigma^{(k)}_T\big)^{1/2}\big(\wh{\mu}-\big(\Sigma^{(k)}_T\big)^{-1/2}\mu_T^*\big)\big\|_2 
\le\big\|\Sigma^{(k)}_T\big\|_2^{1/2}\cdot\|\wh{\mu}-\mu_Y\|_2\\
&\le\sqrt{c'+\delta_0}\cdot\|\wh{\mu}-\mu_Y\|_2 
= O(\sqrt{\epsilon\delta}+\epsilon\log(1/\epsilon)) \;.
\end{align*}
Then we apply Lemma~\ref{lem:mean-to-parameter} by taking 
$\big(\big(\Sigma_T^{(k)}\big)^{1/2}\wh{\mu},\;O\big(\sqrt{\epsilon\delta}+\epsilon\log(1/\epsilon)\big),\zeta/2\big)$ 
as input, to obtain a vector $\theta^{(k+1)}\in\Omega$ such that 
$\|\theta^{(k+1)} - \theta^* \|_2 \le O\big(\big\|\big(\Sigma^{(k)}_T\big)^{1/2} \wh{\mu}-\mu_T^*\big\|_2\big)
\le O(\sqrt{\epsilon\delta}+\epsilon\log(1/\epsilon))$.
This completes the proof.
\end{proof}

\section{Robustly Learning Ising Models with Zero External Field} \label{sec:ising}

In this section, we prove Theorem~\ref{thm:main-zero-ext},
giving our efficient robust learning algorithm for Ising models without
external field under Dobrushin's condition.

Throughout this section, we assume that the target distribution is an Ising model satisfying the Dobrushin condition 
for some \emph{fixed} constant $\eta>0$. Therefore, we will suppress any possible dependence 
on $\eta$ in our asymptotic notation in this section.

For the zero external field case, the probability density function of an Ising model is of the form
$P_\theta(x)=\frac{1}{Z(\theta)}\exp((1/2)\,\sum_{i,j\in[d]}\theta_{ij}x_ix_j)$,
where \new{$(\theta_{ij})_{i,j\in[d]}$} is a $d \times d$ real symmetric matrix with zero diagonal
and $Z(\theta)$ is the \new{partition function}. 
By definition, $P_{\theta}$ is an exponential family with sufficient statistics
$T(x)=(x_ix_j)_{1\le i<j\le d}$ and the projection of $T(x)$ on a fixed direction is $X^TAX$,
where $A\in\mathbb{R}^{d\times d}$ is a symmetric matrix with zero diagonal and $\|A\|_F^2=1/2$.


As already mentioned, we view the \new{Ising model} distribution as an instance of
a general exponential family and apply Algorithm~\ref{main-algorithm}.
The challenge lies in proving correctness. Let $\Omega$ be the set of all $\theta$
such that $P_{\theta}$ satisfies Dobrushin's condition.
We will show that Condition~\ref{exp-family-cond} is satisfied for $\Omega$,
and therefore Algorithm~\ref{main-algorithm} succeeds in our context.

First note that, by our choice of $\Omega$, its diameter is bounded ($\mathrm{diam}(\Omega)=\poly(d)$),
and we can efficiently compute the projection of any point $z \in \R^{d\times(d-1)/2}$.
Moreover, by Fact~\ref{fact:Glauber-dynamics}, we can efficiently approximately sample
from Ising models satisfying Dobrushin's condition.

It remains to verify the first two statement of Condition~\ref{exp-family-cond}.
For the second statement, we need the following sub-exponential concentration inequality
for quadratic functions of $(1-\eta,\alpha)$-bounded Ising models.
This inequality will also be needed for the non-zero external field case.

\begin{lemma}\label{lem:conc-Ising}
Let $X\sim P_\theta$ be an Ising model satisfying Dobrushin's condition
and $\max_{i\in[d]}|\theta_i|\le\alpha$,
where $\alpha>0$ is an absolute constant.
Let $A\in\R^{d\times d}$ be a symmetric matrix
with zero diagonal and $b\in\R^d$.
For any $x\in\{\pm1\}^d$, define $f(x)=(x-v)^TA(x-v)+b^Tx$, 
where $v$ satisfies $\left\|v-\E[X]\right\|_2\le\delta$,
for some constant $\delta>0$.
Then there is a universal constant $c>0$ such that
\begin{align*}
\Pr[|f(X)-\E[f(X)]|>t]\le2\exp\left(-\frac{ct}{(\|A\|_F^2+\|b\|_2^2)^{1/2}}\right).
\end{align*}
\end{lemma}

\noindent Lemma~\ref{lem:conc-Ising} can be derived via machinery
developed in~\cite{gotze2019higher} (see Appendix~\ref{app:lem:conc-Ising} for the proof).
From Lemma~\ref{lem:conc-Ising}, it follows that the sufficient statistics $T(x)$
has sub-exponential tails, for some universal constant $c>0$.

It remains to verify the first statement of Condition~\ref{exp-family-cond}, i.e.,
to show that for any $\theta\in\Omega$ the \new{Ising model} distribution $P_{\theta}$
satisfies $\cov_{X\sim P_{\theta}}[T(X)]\succeq c'\,I$, for some universal constant $c'>0$.
\new{Equivalently, it suffices to} show that for any unit vector $w\in\mathbb{S}^{d\times(d-1)/2-1}$,
it holds $$w^T\cov_{X\sim P_\theta}[T(X)]w=\var_{X\sim P_\theta}[w^TT(X)] \new{\geq} c' \;.$$
\new{We start with some very basic intuition about this statement.}
Note that in the very special case where $\theta_{ij}=0,\forall i,j\in[d]$,
$X\sim P_\theta$ is the uniform distribution on the hypercube, i.e., its coordinates
are independent Rademacher random variables. In this case, it is easy to see that
for any symmetric matrix $A\in\mathbb{R}^{d\times d}$ we have that
$\var[X^TAX] =2 \littlesum_{i\ne j}A_{ij}^2$.
Intuitively, for any $(M,\alpha)$-bounded Ising model (possibly containing a non-zero external field)
for some constants $M,\alpha>0$, the entries of $X$ are {\em nearly} independent,
which allows us to prove the desired variance lower bound.

Our result in this context is the following theorem, which may be of independent
interest.

\begin{theorem}\label{thm:ac}
Let $X \sim P_{\theta}$ be an $(M,\alpha)$-bounded Ising model 
(possibly with non-zero external field), for some constants $M, \alpha>0$. 
There is a constant $c(M,\alpha)>0$ such that for any symmetric matrix 
$A \in \R^{d\times d}$ with zero diagonal and any $v\in\R^d$, we have that
$$\var[(X-v)^TA(X-v)]\ge c(M,\alpha)\|A\|_F^2 \;.$$
\end{theorem}

\noindent
Before proving the theorem, we provide a brief outline of the proof. By definition, we can write
\begin{align*}
\var[(X-v)^TA(X-v)]&=\frac{1}{2}\E\big[\big((X-v)^TA(X-v)-(Y-v)^TA(Y-v)\big)^2\big]\notag\\
&=\frac{1}{2}\E\big[\big((X-Y)^TA(X+Y-2v)\big)^2\big] \;.
\end{align*}
Since there are dependencies between each $X_i$ and $Y_i$,
it is not easy to bound from below the expectation of the quadratic form directly.
By Fact~\ref{fact:linear-anti-concentration}, we know that \new{$\cov[X]\succeq c'(M,\alpha)\,I$},
for some universal constant \new{$c'(M,\alpha)>0$}. A natural idea is to reduce
the original problem to lower bounding the variance of a linear form.

Define the random variables $S=\{i\in[d]\mid X_i=Y_i\}$ and
$A^S_{ij}=A_{ij}, \forall i\notin S,j\in[d],\;W^S_i=\mathbb{I}[i\in S]X_i-v_i, \forall i\in[d]$.
The key observation is that conditioning on a fixed set $S$,
the marginal distributions of $X_S$ and $X_{-S}$ are
{\em independent} $(2M,2\alpha)$-bounded Ising model distributions.
In addition, conditioning on a fixed set $S$, $X-Y$ only depends
on $X_{-S}$, and $X+Y-2v$ only depends on $X_S$.
Therefore, we can write
\begin{align*}
&\quad\E_{(X,Y)}\big[\big((X-Y)^TA(X+Y-2v)\big)^2\mid S\big]\\
&=\E_{(X_S,X_{-S})}\big[\big(X^T_{-S}A^SW^S\big)^2\mid S\big]=\E_{X_S}\big[\E_{X_{-S}}\big[\big(X^T_{-S}A^SW^S\big)^2\mid X_S,S\big]\mid S\big]\\
&\ge\E_{X_S}\big[\lambda_\mathrm{min}\big(\E_{X_{-S}}\big[X_{-S}X_{-S}^T\mid S\big]\big)\|A^SW^S\|_2^2\mid S\big]\\
&\ge c'(M,\alpha)\E\big[\|A^SW^S\|_2^2\mid S\big] \;,
\end{align*}
where $\lambda_\mathrm{min}\big(\E_{X_{-S}}\big[X_{-S}X_{-S}^T\mid S\big]\big)$ denotes the minimum eigenvalue of $\E_{X_{-S}}\big[X_{-S}X_{-S}^T\mid S\big]$.
Given this, we can express $\E\big[\|A^SW^S\|_2^2\mid S\big]$
in terms of the variance of a linear form,
and apply Fact~\ref{fact:linear-anti-concentration} again
to obtain the desired lower bound.

We can now proceed with the formal proof.

\begin{proof}[Proof of Theorem~\ref{thm:ac}]
By definition, we have that
\begin{align}\label{eq:ac1}
\var[(X-v)^TA(X-v)]&=\frac{1}{2}\E\left[\left((X-v)^TA(X-v)-(Y-v)^TA(Y-v)\right)^2\right]\notag\\
&=\frac{1}{2}\E\left[\left((X-Y)^TA(X+Y-2v)\right)^2\right] \;,
\end{align}
where $Y$ is an independent copy of $X$. Let $S=\{i\in[d]\mid X_i=Y_i\}$. Then we can write
\begin{align*}
\E\left[\left((X-Y)^TA(X+Y-2v)\right)^2\right]
&=16 \E\left[\E\Bigg[\Bigg(\sum_{i\notin S}{X_i\Bigg(\sum_{j\in S}{A_{ij}(X_j-v_j)}-\sum_{j\notin S}{v_jA_{ij}}\Bigg)}\Bigg)^2\Bigg]\Bigg| S\right]\\
&=16\E\left[\E\left[\left(X^T_{-S}A^SW^S\right)^2\mid S\right]\right] \;,
\end{align*}
where $A^S_{ij}=A_{ij}$ for all $i\notin S,j\in[d]$ and $W^S_i=\mathbb{I}[i\in S]X_i-v_i$, for all $i\in[d]$.

Now for a fixed subset $S\subseteq[d]$, we calculate the conditional probability $\Pr[X=x\mid S]$. 
By our definition of $S$, we have that
\begin{align*}
\Pr[X=x\mid S]
&=\Pr[X=x\wedge Y_S=x_S\wedge Y_{-S}=-x_{-S}\mid S]\\
&=\frac{\Pr[[X=x\wedge Y_S=x_S\wedge Y_{-S}=-x_{-S}]}{\Pr[S]}\\
&=\frac{\exp\Bigg(2\sum_{i\in S,j\in S}{\theta_{ij}x_ix_j}+2\sum_{i\notin S,j\notin S}{\theta_{ij}x_ix_j}
+ 2\sum_{i\in S}{\theta_ix_i}\Bigg)}{Z(\theta)^2\Pr[S]} \;,
\end{align*}
where $Z(\theta)$ is the partition function of Ising model $P_\theta$. 
Therefore conditioning on $S$, the marginal distribution of $X$ is exactly an Ising model distribution
with parameters
\begin{align*}
\theta^S_{ij}=
\begin{cases}
2\theta_{ij}& i\in S,j\in S\text{ or } i \notin S, j \notin S,\\
0&\text{ otherwise} \;,
\end{cases}
\qquad\text{and}\qquad 
\theta^S_i=
\begin{cases}
2\theta_i & i\in S,\\
0&i\notin S,
\end{cases}
\end{align*}
which implies that conditioning on $S$, the marginal distribution $X_S$ and $X_{-S}$ 
are independent $(2M,2\alpha)$-bounded Ising model distributions and $\E[X_{-S}\mid S]=0$. 
Therefore, from Fact~\ref{fact:linear-anti-concentration}, 
there is a universal constant $c_1(M,\alpha)>0$ such that
\begin{align*}
\E\left[\left(X^T_{-S}A^SW^S\right)^2\mid S\right]
&=\E_{X_S}\left[\E_{X_{-S}}\left[\left(X^T_{-S}A^SW^S\right)^2\mid X_S,S\right]\mid S\right]\\
&\ge\E_{X_S}\left[\lambda_\mathrm{min}\left(\E_{X_{-S}}\left[X_{-S}X_{-S}^T\mid S\right]\right)\|A^SW^S\|_2^2\mid S\right]\\
&\ge c_1(M,\alpha)\E\left[\|A^SW^S\|_2^2\mid S\right] \;,
\end{align*}
where $\lambda_\mathrm{min}\left(\E_{X_{-S}}\left[X_{-S}X_{-S}^T\mid S\right]\right)$ 
denotes the minimum eigenvalue of $\E_{X_{-S}}\left[X_{-S}X_{-S}^T\mid S\right]$ 
and in the first inequality, we use the fact that conditioning on $S$, 
$X_S$ and $X_{-S}$ are independent.
Therefore, we have that
\begin{align}\label{eq:ac2}
&\quad\E\left[\left((X-Y)^TA(X+Y-2v)\right)^2\right]
=16\E\left[\E\left[\left(X^T_{-S}A^SW^S\right)^2\mid S\right]\right]\\
&\ge16c_1(M,\alpha)\E\left[\E\left[\|A^SW^S\|_2^2\mid S\right]\right]\notag 
=16c_1(M,\alpha)\E\left[\E\Bigg[\sum_{i\notin S}\Bigg(\sum_{j\in[d]}{a_{ij}(\mathbb{I}[j\in S]X_j-v_j)}\Bigg)^2\Bigg|S\Bigg]\right]\notag\\
&=4c_1(M,\alpha)\E\left[\|A'(X+Y-2v)\|_2^2\right] \;,
\end{align}
where $A'_{ij}=\mathbb{I}[X_i\ne Y_i]A_{ij}$ for all $i,j\in[d]$. 
Now we write $A'=[\mathbb{I}[X_1\ne Y_1]a^1,\cdots,\mathbb{I}[X_d\ne Y_d]a^d]^T$, 
where $(a^i)^T$ denotes the $i$-th row vector of matrix $A$. 
By linearity of expectation, we have that
\begin{align}\label{eq:ac3}
\E\left[\|A'(X+Y-2v)\|_2^2\right]=\sum_{i=1}^{d}{\E\left[\langle \mathbb{I}[X_i\ne Y_i]a^i,X+Y-2v\rangle^2\right]} \;.
\end{align}
By the law of total expectation, we can write
\begin{align}\label{eq:ac4}
&\quad\E\left[\langle\mathbb{I}[X_i\ne Y_i]a^i,X+Y-2v\rangle^2\right]\\
&=\Pr[X_i=1,Y_i=-1] \cdot \E\left[\langle a^i,X + Y-2v\rangle^2\mid X_i=1,Y_i=-1\right]\notag\\
&\quad+\Pr[X_i=-1,Y_i=1] \cdot \E\left[\langle a^i,X+Y-2v\rangle^2\mid X_i=-1,Y_i=1\right]\notag\\
&\ge2\left(\frac{\exp(-2(\alpha+M))}{1+\exp(-2(\alpha+M))}\right)^2\E\left[\langle a^i,X + Y-2v\rangle^2\mid X_i=1,Y_i=-1\right]\;,
\end{align}
where the inequality comes from Fact~\ref{fact:marginal} and the fact that $Y$ is an independent copy of $X$.

Now we try to bound $\E\left[\langle a^i,X+Y-2v\rangle^2\mid X_i=1,Y_i=-1\right]$. 
From Fact~\ref{fact:cond}, conditioning on $X_i=q\in\{\pm1\}$, 
$X_{-i}$ is an Ising model with parameter $\theta'$ satisfying the following property
\begin{align*}
\max_{j\in[d]\setminus\{i\}}{|\theta'_j|}\le M+\alpha,\qquad\max_{j\in[d]\setminus\{i\}}{\sum_{k\in[d]\setminus\{i,j\}}|\theta'_{jk}|}\le M \;,
\end{align*} 
which implies that conditioning on $X_i=q$, $X_{-i}$ is an $(M,M+\alpha)$-bounded Ising model.
Note that $X_{-i}$ and $Y_{-i}$ are independent, 
conditioning on $X_i=-1,Y_i=1$, from Fact~\ref{fact:linear-anti-concentration}, 
there is a constant $c_2(M,\alpha)>0$ such that
\begin{align}\label{eq:ac5}
\E\left[\langle a^i,X+Y-2v\rangle^2\mid X_i=1,Y_i=-1\right]
&\ge \var\left[\langle a^i,X+Y\rangle\mid X_i=1,Y_i=-1\right]\notag\\
&=\var[\langle a^i,X\rangle\mid X_i=1]+\var[\langle a^i,Y\rangle\mid Y_i=-1]\notag\\
&\ge c_2(M,\alpha)\|a^i\|_2^2 \;,
\end{align}
Combining~\eqref{eq:ac3},~\eqref{eq:ac4} and~\eqref{eq:ac5}, we obtain that
\begin{align}\label{eq:ac6}
\E\left[\|A'(X+Y-2v)\|_2^2\right]
&=\sum_{i=1}^{d}{\E\left[\langle\mathbb{I}[X_i\ne Y_i]a^i,X+Y-2v\rangle^2\right]}\notag\\
&\ge2\left(\frac{\exp(-2(\alpha+M))}{1+\exp(-2(\alpha+M))}\right)^2\sum_{i=1}^{d}\E\left[\langle a^i,X + Y-2v\rangle^2\mid X_i=1,Y_i=-1\right]\notag\\
&\ge2\left(\frac{\exp(-2(\alpha+M))}{1+\exp(-2(\alpha+M))}\right)^2c_2(M,\alpha)\|A\|_F^2 \;.
\end{align}
Combining~\eqref{eq:ac1},~\eqref{eq:ac2} and~\eqref{eq:ac6}, we get that there exists a constant $c(M,\alpha)>0$ such that
\begin{align*}
\var[(X-v)^TA(X-v)]=\frac{1}{2}\E\left[\left((X-Y)^TA(X+Y-2v)\right)^2\right]\ge c(M,\alpha)\|A\|_F^2 \;.
\end{align*}
This completes the proof.
\end{proof}

We are now ready to prove the main result of this section.

\smallskip

\begin{proof}[Proof of Theorem~\ref{thm:main-zero-ext}]
Let 
$$\Omega=\left\{ (\theta_{ij})_{1\le i<j\le d}\in\R^{d\times(d-1)/2}\mid\max_{i\in[d]}\littlesum_{j=1}^{i-1}|\theta_{ji}|+\littlesum_{j=i+1}^{d}|\theta_{ij}|\le1-\eta \right\} \;,$$
where $\eta>0$ is the constant in Definition~\ref{def:high-temp}.
Let $\theta\in\Omega$ and $P_{\theta}$ be the corresponding Ising model distribution. By definition, $P_{\theta}$\new{\footnote{For simplicity, we also use $\theta$ to denote the $d\times d$ symmetric matrix with zero diagonal.}} is an exponential family with sufficient statistics $T(x)=(x_ix_j)_{1\le i<j\le d}$. 
In order to apply Algorithm~\ref{main-algorithm}, we check each statement in 
Condition~\ref{exp-family-cond} one by one.
By our choice of $\Omega$, we know that $\mathrm{diam}(\Omega)=O(d)$ 
and we can efficiently compute the projection of any point $z\in\mathbb{R}^{d\times(d-1)/2}$.
From Fact~\ref{fact:Glauber-dynamics}, we can sample from $P_{\theta}$ 
within total variation distance $\gamma$ in time $O(d(\log d+\log(1/\gamma)))$, for any $\gamma>0$. 
Therefore, the third statement holds.
From Lemma~\ref{lem:conc-Ising}, there is a universal constant $c>0$ such that for any symmetric matrix $A\in\mathbb{R}^{d\times d}$ with zero diagonal and any $t>0$, we have that
$\Pr_{X\sim P_{\theta}}[|X^TAX-\E[X^TAX]|>t]\le2\exp\left(-(ct)/\|A\|_F\right)$,
which implies the second statement in Condition~\ref{exp-family-cond}.
Moreover, \new{by} Theorem~\ref{thm:ac}, we know that there is a universal constant $c'>0$ 
such that for any symmetric matrix $A\in\mathbb{R}^{d\times d}$ with zero diagonal, we have that
$\var[X^TAX]\ge c'\|A\|_F^2$, which implies the first statement in Condition~\ref{exp-family-cond}.
Therefore, by Theorem~\ref{thm:main-exp-family}, we can efficiently obtain an estimate 
$\wh{\theta}\in\Omega$ such that 
\new{$\dtv(P_{\wh{\theta}},P_{\theta^*})\le O(\|\wh{\theta}-\theta^*\|_F)\le O(\epsilon\log(1/\epsilon))$ 
with probability at least $99/100$. In addition, by our algorithm $\wh{\theta}\in\Omega$, 
and thus the output hypothesis satisfies Dobrushin's condition.}
\end{proof}

\section{Robustly Learning Ising Models with Non-zero External Field}\label{sec:ising-external}

In this section, we provide an efficient algorithm that robustly learns Ising models
with non-zero external field under certain technical assumptions. Specifically, we show:

\begin{theorem}\label{thm:ising-ext}
Let $P_{\theta^*}$ be an Ising model with $\max_{i\in[d]}\sum_{j\ne i}|\theta^*_{ij}|\le M$ and $\max_{i\in[d]}{|\theta^*_i|}\le\alpha$ for some $0\le M<1$ and $\alpha\ge0$. Let $0<\epsilon\le\epsilon_0$ and $S'$ be an $\epsilon$-corrupted set of samples from $P_{\theta^*}$. Furthermore, assume that for some $c_0>0$ and some $c_1>0$ sufficiently large 
\begin{align}\label{eq:non-zero-constraint}
4\left(\frac{M}{1-M}+c_1\sqrt{\epsilon}\right)^2\le(1-c_0)\left(8\left(\frac{\exp(-2(\alpha+2M))}{1+\exp(-2(\alpha+2M))}\right)^2-\frac{2M}{1-M}-c_0\right)
\end{align}
holds. Let $N$ be the size of $S'$.
\new{Then there is a $\poly(N,d)$ time algorithm that, for some $N=\wt{O}_{\alpha,M,c_0}(d^2/\epsilon^2)$}, on input $S'$ and $\epsilon$, returns an Ising model $P_{\wh{\theta}}$ such that with probability at least $99/100$, we have that $\dtv(P_{\wh{\theta}},P_{\theta^*})\le O_{\alpha,M,c_0}(\epsilon\log(1/\epsilon))$. In addition, $P_{\wh{\theta}}$ satisfies the Dobrushin's condition.
\end{theorem}

Notice that the left-hand side of~\eqref{eq:non-zero-constraint} is increasing in $M$ 
and the right-hand side of~\eqref{eq:non-zero-constraint} is decreasing in $M, \alpha$.
Intuitively speaking, as long as the dependencies among each point and the external fields are sufficiently small, we can robustly learn the Ising model distribution in total variation distance. 

For simplicity, we will suppress any possible dependence on $M, \alpha, c_0$ 
in our asymptotic notation in the remaining part of this section.

\medskip

Similar to the zero external field case, to prove Theorem~\ref{thm:ising-ext}, 
we view the Ising model distribution $P_{\theta}$ 
as an instance of an exponential family and apply Algorithm~\ref{main-algorithm}. 
However, if we choose the sufficient statistics $T(x)=((x_ix_j)_{1\le i<j\le d},(x_i)_{i\in[d]})$ in the straightforward way,  the first statement in Condition~\ref{exp-family-cond} will not hold. 
For instance, consider the Ising model $P_{\theta}$ with $\theta_{ij}=0,\forall i,j\in[d]$, 
and $\theta_i=\beta,\forall i\in[d]$, for some $\beta>0$, 
such that $\E_{X\sim P_\theta}[X_i]=1/2, \forall i\in[d]$. 
Let $A\in\R^{d\times d}$ be such that $A_{ij}=\frac{1}{\sqrt{d(d-1)(d+1)}},\forall i\ne j$, 
and $b_i=-\sqrt{\frac{d-1}{d(d+1)}},\forall i\in[d]$. 
In this case, we have that $2\|A\|_F^2+\|b\|_2^2=1$, and
\begin{align*}
\var[X^TAX+b^TX]&=\var[(X-v)^TA(X-v)+(2Av+b)^TX]\\
&=\var[(X-v)^TA(X-v)]\\&\le c\|A\|_F^2=\frac{c}{d+1} \;,
\end{align*}
where the last inequality follows from Fact~\ref{fact:var-ub} and $c>0$ is an absolute constant. 

To address this issue, we rewrite the density of an Ising model as the following ``$v$-centered form''.
Let $v\in\R^d$ be an arbitrary fixed vector. By definition of the Ising model, we have that
\begin{align*}
P_\theta(x)
&=\frac{1}{Z(\theta)}\exp\left(\frac{1}{2}\sum_{i,j\in[d]}{\theta_{ij}x_ix_j}+\sum_{i=1}^{d}{\theta_ix_i}\right)\\
&=\frac{1}{Z(\theta)}\exp\left(\frac{1}{2}\sum_{i,j\in[d]}{\theta_{ij}(x_i-v_i)(x_j-v_j)}
+\sum_{i=1}^{d}{\Bigg(\theta_i+\sum_{j\in[d]}{\theta_{ij}v_j}\Bigg)x_i}\right)\\
&=\frac{1}{Z(\theta)}\exp\left(\frac{1}{2}(x-v)^TJ(\theta)(x-v)+h(\theta)^Tx\right) \;,
\end{align*}
where $J(\theta)_{ij}=\theta_{ij},\forall i,j\in[d]$ and $h(\theta)_i=\theta_i+\sum_{j\in[d]}{\theta_{ij}v_j}$. 
If we write the probability density function  $P_\theta(x)$ in the ``$v$-centered form'' 
as an instance of an exponential family, the sufficient statistics $T(x)$ will be 
$$T(x)=((x_i-v_i)(x_j-v_j)_{1\le i<j\le d},(x_i)_{1\le i\le d}) \;,$$ 
and the projection of $T(x)$ on a fixed direction is
$$(X-v)^TA(X-v)+b^TX \;,$$ 
where $A\in\R^{d\times d}$ is a symmetric matrix 
with zero diagonal and $b\in\R^d$ with $2\|A\|_F^2+\|b\|_2^2=1$.

In this way, by taking $v$ to be an estimate of $\E_{X\sim P_{\theta^*}}[X]$, 
we are able to prove the following lower bound on the covariance 
of the sufficient statistics $T(x)$, and then apply Algorithm~\ref{main-algorithm} 
to robustly learn the parameter $J(\theta^*)$ and $h(\theta^*)$ in the ``$v$-centered form''.

Our anti-concentration result for the non-zero external field case is the following:

\begin{theorem}\label{thm:ac-2}
Let $X\sim P_{\theta}$ be an Ising model with $\max_{i\in[d]}\sum_{j\ne i}|\theta_{ij}|\le M$ 
and $\max_{i\in[d]}{|\theta_i|}\le\alpha$, for some $0\le M<1$ and $\alpha\ge0$.
Let $v\in\R^d$ be a vector such that $\|v-\E[X]\|_2\le\delta$, for some $\delta>0$. 
If there is a constant $c_0>0$ such that
\begin{align*}
4\left(\frac{M}{1-M}+\delta\right)^2\le(1-c_0)\left(8\left(\frac{\exp(-2(\alpha+2M))}{1+\exp(-2(\alpha+2M))}\right)^2-\frac{2M}{1-M}-c_0\right) \;,
\end{align*}
then there exists another constant $c(\alpha,M,c_0)>0$ such that
\begin{align*}
\var[(X-v)^TA(X-v)+b^TX]\ge c(\alpha,M,c_0)(\|A\|_F^2+\|b\|_2^2)
\end{align*}
holds for all symmetric matrix $A\in\R^{d\times d}$ with zero diagonal and $b\in\R^d$.
\end{theorem}

\begin{proof}
By definition, we have that
\begin{align}\label{eq:ac7}
\var[(X-v)^TA(X-v)+b^TX]&=\frac{1}{2}\E\left[\left((X-v)^TA(X-v)-(Y-v)^TA(Y-v)+b^TX-b^TY\right)^2\right]\notag\\
&=\frac{1}{2}\E\left[\left((X-Y)^T\left(A(X+Y-2v)+b\right)\right)^2\right] \;,
\end{align}
where $Y$ is an independent copy of $X$. Let $S=\{i\in[d]\mid X_i=Y_i\}$ and we can write
\begin{align*}
&\quad\E\left[\left((X-Y)^T\left(A(X+Y-2v)+b\right)\right)^2\right]\\
&=\E\left[\E\Bigg[\Bigg(\sum_{i\notin S}{X_i\Bigg(b_i+\sum_{j\in S}{2a_{ij}(X_j-v_j)}-\sum_{j\notin S}{2v_ja_{ij}}\Bigg)}\Bigg)^2\Bigg]\Bigg| S\right]\\&=4\E\left[\E\left[\left(X^T_{-S}\left(A^SW^S+b_{-S}\right)\right)^2\mid S\right]\right],
\end{align*}
where $A^S_{ij}=A_{ij}$ for all $i\notin S,j\in[d]$ and $W^S_i=2\left(\mathbb{I}[i\in S]X_i-v_i\right)$ for all $i\in[d]$.

Now for a fixed subset $S\subseteq[d]$, we calculate the conditional probability $\Pr[X=x\mid S]$. By definition of $S$, we have that
\begin{align*}
\Pr[X=x\mid S]&=\Pr[X=x\wedge Y_S=x_S\wedge Y_{-S}=-x_{-S}\mid S]\\&=\frac{\Pr[[X=x\wedge Y_S=x_S\wedge Y_{-S}=-x_{-S}]}{\Pr[S]}\\&=\frac{\exp\Bigg(2\sum_{i\in S,j\in S}{\theta_{ij}x_ix_j}+2\sum_{i\notin S,j\notin S}{\theta_{ij}x_ix_j}+2\sum_{i\in S}{\theta_ix_i}\Bigg)}{Z(\theta)^2\Pr[S]},
\end{align*}
where $Z(\theta)$ is the partition function of Ising model $P_\theta$. Therefore conditioning on $S$, the marginal distribution of $X$ is exactly an Ising model distribution
with parameters
\begin{align*}
\theta^S_{ij}=
\begin{cases}
2\theta_{ij}& i\in S,j\in S\text{ or } i \notin S, j \notin S,\\
0&\text{ otherwise} \;,
\end{cases}
\qquad\text{and}\qquad 
\theta^S_i=
\begin{cases}
2\theta_i & i\in S,\\
0&i\notin S,
\end{cases}
\end{align*}
which implies that conditioning on $S$, the marginal distribution $X_S$ and $X_{-S}$ are independent $(2M,2\alpha)$-bounded Ising model distributions and $\E[X_{-S}\mid S]=0$. 
Therefore, from Fact~\ref{fact:linear-anti-concentration}, there is a constant $c_1(M,\alpha)>0$ such that
\begin{align*}
\E\left[\left(X^T_{-S}\left(A^SW^S+b_{-S}\right)\right)^2\mid S\right]&=\E_{X_S}\left[\E_{X_{-S}}\left[\left(X^T_{-S}\left(A^SW^S+b_{-S}\right)\right)^2\mid X_S,S\right]\mid S\right]\\&\ge\E_{X_S}\left[\lambda_\mathrm{min}\left(\E_{X_{-S}}\left[X_{-S}X_{-S}^T\mid S\right]\right)\|A^SW^S+b_{-S}\|_2^2\mid S\right]\\&\ge c_1(M,\alpha)\E\left[\|A^SW^S+b_{-S}\|_2^2\mid S\right],
\end{align*}
where $\lambda_\mathrm{min}\left(\E_{X_{-S}}\left[X_{-S}X_{-S}^T\mid S\right]\right)$ denotes the minimum eigenvalue of $\E_{X_{-S}}\left[X_{-S}X_{-S}^T\mid S\right]$ and in the first inequality, we use the fact that conditioning on $S$, $X_S$ and $X_{-S}$ are independent. Therefore, we have that
\begin{align}\label{eq:ac8}
&\quad\E\left[\left((X-Y)^T\left(A(X+Y-2v)+b\right)\right)^2\right]\\
&=4\E\left[\E\left[\left(X^T_{-S}\left(A^SW^S+b_{-S}\right)\right)^2\mid S\right]\right]\notag\\
&\ge4c_1(M,\alpha)\E\left[\E\left[\|A^SW^S+b_{-S}\|_2^2\mid S\right]\right]\\
&=c_1(M,\alpha)\E\left[\E\Bigg[\sum_{i\notin S}\Bigg(b_i+\sum_{j\in[d]}{2A_{ij}(\mathbb{I}[j\in S]X_j-v_j)}\Bigg)^2\Bigg|S\Bigg]\right]\notag\\
&=c_1(M,\alpha)\E\left[\|A'(X+Y-2v)+b'\|_2^2\right],
\end{align}
where $A'_{ij}=\mathbb{I}[X_i\ne Y_i]A_{ij}$ for all $i,j\in[d]$ and $b'_i=\mathbb{I}[X_i\ne Y_i]b_i$ for all $i\in[d]$. 
Now we write $A'=[\mathbb{I}[X_1\ne Y_1]a^1,\cdots,\mathbb{I}[X_d\ne Y_d]a^d]^T$, 
where $(a^i)^T$ denotes the $i$-th row vector of matrix $A$. By linearity of expectation, we have that
\begin{align}\label{eq:ac9}
\E\left[\|A'(X+Y-2v)+b'\|_2^2\right]=\sum_{i=1}^{d}{\E\left[\left(\langle\mathbb{I}[X_i\ne Y_i]a^i,X+Y-2v\rangle+b'_i\right)^2\right]}.
\end{align}
Fix some $i\in[d]$. Note that $Y$ is an independent copy of $X$, and we can write
\begin{align}\label{eq:ac10}
&\quad\E\left[\left(\langle\mathbb{I}[X_i\ne Y_i]a^i,X+Y-2v\rangle+b'_i\right)^2\right]\\
&=\Pr[X_i=1,Y_i=-1]\cdot\E\left[\left(\langle a^i,X+Y-2v\rangle+b_i\right)^2\mid X_i=1,Y_i=-1\right]\notag\\
&\quad+\Pr[X_i=-1,Y_i=1]\cdot\E\left[\left(\langle a^i,X+Y-2v\rangle+b_i\right)^2\mid X_i=-1,Y_i=1\right]\notag\\
&\ge2\left(\frac{\exp(-2(\alpha+M))}{1+\exp(-2(\alpha+M))}\right)^2\cdot\E\left[\left(\langle a^i,X+Y-2v\rangle+b_i\right)^2\mid X_i=-1,Y_i=1\right],
\end{align}
where the inequality comes from Fact~\ref{fact:marginal}.

Now we bound $\E\left[\left(\langle a^i,X+Y-2v\rangle+b_i\right)^2\mid X_i=-1,Y_i=1\right]$ as follows.
From Fact~\ref{fact:cond}, we know that conditioning on $X_i=q\in\{\pm1\}$, $X_{-i}$ is an Ising model over $\{\pm1\}^{d-1}$ with parameter $\theta'$ satisfying the following property
\begin{align*}
\max_{j\in[d]\setminus\{i\}}{|\theta'_j|}\le M+\alpha,\qquad\max_{j\in[d]\setminus\{i\}}{\sum_{k\in[d]\setminus\{i,j\}}|\theta'_{jk}|}\le M,
\end{align*}
which implies that conditioning on $X_i=q$, $X_{-i}$ is an $(M,M+\alpha)$-bounded Ising model.
Let $\mu_{-i}^1$ denote the conditional expectation over $x_{-i}$ conditioning on $x_i=1$ and $\mu_{-i}^{-1}$ denote the conditional expectation over $x_{-i}$ conditioning on $x_i=-1$. Note that $X_{-i}$ and $Y_{-i}$ are independent conditioning on $X_i=-1,Y_i=1$, we have that
\begin{align*}
&\quad\E\left[\left(\langle a^i,X+Y-2v\rangle+b_i\right)^2\mid X_i=1,Y_i=-1\right]\\&=\var\left[\langle a^i,X+Y-2v\rangle+b_i\mid X_i=1,Y_i=-1\right]+\E\left[\langle a^i,X+Y-2v\rangle+b_i\mid X_i=1,Y_i=-1\right]^2\\&=\var[\langle a^i_{-i},X_{-i}\rangle\mid X_i=1]+\var[\langle a^i_{-i},Y_{-i}\rangle\mid Y_i=-1]+\left(b_i+\langle a^i_{-i},\mu_{-i}^1+\mu_{-i}^{-1}-2v_{-i}\rangle\right)^2\\&\ge\var[\langle a^i_{-i},X_{-i}\rangle\mid X_i=1]+\var[\langle a^i_{-i},Y_{-i}\rangle\mid Y_i=-1]+b_i^2+2b_i\langle a^i_{-i},\mu_{-i}^1+\mu_{-i}^{-1}-2v_{-i}\rangle\\&\ge\var[\langle a^i_{-i},X_{-i}\rangle\mid X_i=1]+\var[\langle a^i_{-i},Y_{-i}\rangle\mid Y_i=-1]+b_i^2-2|b_i|\|a^i\|_2\|\mu_{-i}^1+\mu_{-i}^{-1}-2v_{-i}\|_2,
\end{align*}
where we use $A_{ii}=0,\forall i\in[d]$.

Let $\mu=\E[X]$ and thus $\|\mu-v\|_2\le\delta$ by our assumption. From Fact~\ref{fact:conditional-mean-distance}, we know that
\begin{align*}
\|\mu_{-i}^1+\mu_{-i}^{-1}-2v_{-i}\|_2&\le\|\mu_{-i}^1+\mu_{-i}^{-1}-2\mu_{-i}\|_2+2\|\mu_{-i}-v_{-i}\|_2\\&=(1-\Pr[X_i=1])\|\mu_{-i}^1-\mu_{-i}^{-1}\|_2+2\|\mu_{-i}-v_{-i}\|_2\\&\le\|\mu_{-i}^{1}-\mu_{-i}^{-1}\|_1+2\delta\\&\le\frac{2M}{1-M}+2\delta.
\end{align*}
From Fact~\ref{fact:conditional-mean-distance} and Fact~\ref{fact:cond}, we have that
\begin{align*}
&\quad\var[\langle a^i_{-i},X_{-i}\rangle\mid X_i=q]=\sum_{j\ne i,k\ne i}A_{ij}A_{ik}\cov(X_j,X_k\mid X_i=q)\\&\ge\sum_{j\ne i}{A_{ij}^2\var[X_j\mid X_i=q]}-\sum_{j\ne i,k\ne i,k\ne j}|A_{ij}||A_{ik}||\cov(X_j,X_k\mid X_i=q)|\\&\ge\sum_{j\ne i}{A_{ij}^2\var[X_j\mid X_i=q]}-\sum_{j\ne i,k\ne i,k\ne j}\frac{(A_{ij}^2+A_{ik}^2)|\cov(X_j,X_k\mid X_i=q)|}{2}\\&=\sum_{j\ne i}A_{ij}^2\Bigg(\var[X_j\mid X_i=q]-\sum_{k\ne j,k\ne i}|\cov(X_j,X_k\mid X_i=q)|\Bigg)\\&\ge\sum_{j\ne i}{A_{ij}^2\left(\var[X_j\mid X_i=q]-\frac{M}{1-M}\right)}\\&\ge\left(4\left(\frac{\exp(-2(\alpha+2M))}{1+\exp(-2(\alpha+2M))}\right)^2-\frac{M}{1-M}\right)\|a^i\|_2^2.
\end{align*}
Therefore, we have that
\begin{align}\label{eq:ac11}
&\quad\E\left[\left(\langle a^i,X+Y-2v\rangle+b_i\right)^2\mid X_i=1,Y_i=-1\right]\notag\\
&\ge\var[\langle a^i_{-i},X_{-i}\rangle\mid X_i=1]+\var[\langle a^i_{-i},Y_{-i}\rangle\mid Y_i=-1]+b_i^2-2|b_i|\|a^i\|_2\|\mu_{-i}^1+\mu_{-i}^{-1}-2v_{-i}\|_2\notag\\
&\ge\left(8\left(\frac{\exp(-2(\alpha+2M))}{1+\exp(-2(\alpha+2M))}\right)^2-\frac{2M}{1-M}\right)\|a^i\|_2^2+b_i^2-2|b_i|\|a^i\|_2\left(\frac{2M}{1-M}+2\delta\right)\notag\\
&\ge c_0(\|a^i\|_2^2+b_i^2) \;,
\end{align}
as long as
\begin{align*}
4\left(\frac{M}{1-M}+\delta\right)^2\le(1-c_0)\left(8\left(\frac{\exp(-2(\alpha+2M))}{1+\exp(-2(\alpha+2M))}\right)^2-\frac{2M}{1-M}-c_0\right).
\end{align*}
Combining~\eqref{eq:ac9},~\eqref{eq:ac10} and~\eqref{eq:ac11}, we obtain that
\begin{align}\label{eq:ac12}
&\quad\E\left[\|A'(X+Y-2v)+b'\|_2^2\right]
=\sum_{i=1}^{d}{\E\left[\left(\langle\mathbb{I}[X_i\ne Y_i]\alpha^i,X+Y-2v\rangle+b'_i\right)^2\right]}\notag\\
&\ge 2\left(\frac{\exp(-2(\alpha+M))}{1+\exp(-2(\alpha+M))}\right)^2\sum_{i=1}^{d}\E\left[\left(\langle a^i,X+Y-2v\rangle+b_i\right)^2\mid X_i=-1,Y_i=1\right]\notag\\
&\ge2c_0\left(\frac{\exp(-2(\alpha+M))}{1+\exp(-2(\alpha+M))}\right)^2(\|A\|_F^2+\|b\|_2^2) \;.
\end{align}
Combining~\eqref{eq:ac7},~\eqref{eq:ac8} and~\eqref{eq:ac12}, it follows that there is a constant $c(\alpha,M,c_0)$ such that
\begin{align*}
\var[(X-v)^TA(X-v)+b^TX]&=\frac{1}{2}\E\left[\left((X-Y)^T\left(A(X+Y-2v)+b\right)\right)^2\right]\\
&\ge c(\alpha,M,c_0)(\|A\|_F^2+\|b\|_2^2) \;.
\end{align*}
This completes the proof.
\end{proof}

We can now prove the main theorem of this section.

\begin{proof}[Proof of Theorem~\ref{thm:ising-ext}]
From Fact~\ref{fact:sub-Gaussian}, we know that $X\sim P_{\theta^*}$ is sub-Gaussian, 
from which it follows that $\cov_{X\sim P_{\theta^*}}[X] \preceq c_0 \, I$, 
for some universal constant $c_0>0$. 
Hence, we can apply the robust mean estimation algorithm 
for bounded covariance distributions (Fact~\ref{bounded-cov}) 
to obtain an estimate $v\in\R^d$ with 
$\ltwo{v-\E_{X\sim P_{\theta^*}}[X]}\le c_1\sqrt{\epsilon}\le c_1\sqrt{\epsilon_0}$, 
for some constant $c_1>0$.

Let 
$$\Omega=\left\{(\theta_{ij})_{1\le i<j\le d}\in\R^{d\times(d-1)/2}, (\theta_i)_{i\in[d]}\in\R^d \mid 
\max_{i\in[d]}\littlesum_{j=1}^{i-1}|\theta_{ji}| + \littlesum_{j=i+1}^{d}|\theta_{ij}|\le M, \max_{i\in[d]}|\theta_i|\le\alpha \right\} \;.$$ 
For any $\theta\in\Omega$, define $J(\theta)_{ij}=\theta_{ij},\forall 1\le i<j\le d$, 
and $h(\theta)_i=\theta_i+\sum_{j=1}^{i-1}{\theta_{ji}v_j}+\sum_{j=i+1}^{d}{\theta_{ij}v_j},\forall i\in[d]$. 
Let $\Omega_{J,h}=\{(J(\theta),h(\theta))\mid \theta\in\Omega\}$. 
Note that for any $\theta^1,\theta^2\in\Omega$ and any $0<\lambda<1$, 
we have that 
$$J(\lambda\theta^1+(1-\lambda)\theta^2)=\lambda J(\theta^1)+(1-\lambda)J(\theta^2)$$ 
and 
$$h(\lambda\theta^1+(1-\lambda)\theta^2)=\lambda h(\theta^1)+(1-\lambda)h(\theta^2) \;,$$ 
which implies that $\Omega_{J,h}$ is convex because of the convexity of $\Omega$.
Let $\theta\in\Omega$ and $P_\theta$ be the corresponding Ising distribution. 
We write $P_{\theta}$ in the ``$v$-centered form'', 
i.e., 
$$P_\theta(x)=\frac{1}{Z(\theta)}\exp\left(\frac{1}{2}(x-v)^TJ(\theta)(x-v)+h(\theta)^Tx\right) \;,$$
\footnote{For simplicity, we also use $J(\theta)$ to note the $d\times d$ symmetric matrix with zero diagonal.}
where $Z(\theta)$ is the partition function. In this way, $P_{\theta}$ 
is an exponential family with sufficient statistics $T(x)=((x_i-v_i)(x_j-v_j)_{1\le i<j\le d},(x_i)_{1\le i\le d})$.

Now we check the statements in Condition~\ref{exp-family-cond} one by one 
in order to apply Algorithm~\ref{main-algorithm} to obtain an estimation 
of $J(\theta^*)$ and $h(\theta^*)$.
By our choice of $\Omega_{J,h}$, we know that $\mathrm{diam}(\Omega_{J,h})=O(d)$ 
and we can efficiently compute the projection of any point $z\in\mathbb{R}^{d\times(d-1)/2}$.
From Fact~\ref{fact:Glauber-dynamics}, we can sample from $P_\theta$ 
within total variation distance $\gamma$ in time $O(d(\log d+\log(1/\gamma)))$, for any $\gamma>0$. 
Therefore, the third statement holds.

From Lemma~\ref{lem:conc-Ising}, there is a universal constant $c>0$ 
such that for any symmetric matrix $A\in\R^{d\times d}$ with zero diagonal 
and $b\in\mathbb{R}^d$, we have that
\begin{align*}
\Pr_{X\sim P_{\theta}}\left[|f(X)-\E[f(X)]|>t\right]\le2\exp\left(-\frac{ct}{(\|A\|_F^2+\|b\|_2^2)^{1/2}}\right)\,,
\end{align*}
where $f(x)=(x-v)^TA(x-v)+b^Tx,\forall x\in\{\pm1\}^d$. This implies the second statement in Condtion~\ref{exp-family-cond}.

Moreover, from Theorem~\ref{thm:ac-2}, we know that there is a universal constant $c'>0$ such that for any symmetric matrix $A\in\mathbb{R}^{d\times d}$ with zero diagonal and $b\in\mathbb{R}^d$, we have that
\begin{align*}
\var[(X-v)^TA(X-v)+b^TX]\ge c'(\|A\|_F^2+\|b\|_2^2)\,,
\end{align*}
which implies the first statement in Condition~\ref{exp-family-cond}.

Thus, we can apply Algorithm~\ref{main-algorithm} to obtain estimates 
$\wh{J},\wh{h}$ with $\sqrt{\|\wh{J}-J(\theta^*)\|_F^2+\|\wh{h}-h(\theta^*)\|_2^2}\le O(\epsilon\log(1/\epsilon))$. 
Let $\wh{\theta}_{ij}=\wh{J}_{ij},\forall i, j \in[d]$, and $\wh{\theta}_i=\wh{h}_i-\sum_{j=1}^{d}\wh{J}_{ij}v_j$. 
From Theorem~\ref{thm:main-exp-family}, we have that 
$$\dtv(P_{\wh{\theta}},P_{\theta^*})\le O\left(\sqrt{\|\wh{J}-J(\theta^*)\|_F^2+\|\wh{h}-h(\theta^*)\|_2^2}\right)\le O(\epsilon\log(1/\epsilon)) \;,$$
where $P_{\wh{\theta}}$ denotes the Ising model distribution corresponding to parameter $\wh{\theta}$. 
In addition, by our algorithm, we have that $\max_{i\in[d]}\sum_{j\ne i}|\wh{\theta}_{ij}|\le1-\eta$, 
and thus the output hypothesis satisfies Dobrushin's condition.
This completes the proof.
\end{proof}

\clearpage

\bibliography{allrefs}

\bibliographystyle{alpha}

\newpage

\appendix

\section*{Appendix}

\section{Omitted Technical Preliminaries} \label{app:prelims-1}

\subsection{Dobrushin's uniqueness condition} \label{app:dob}

Here we introduce the original definition of Dobrushin's condition 
through the influence between points in general graphical model.
\begin{definition}[Influence in graphical models]\label{Dobrushin's-condition}
Let $D$ be a distribution over some set of points $V$. 
Let $S_j$ denote the set of \new{state} pairs $(X,Y)$ which differ only at point $j$. 
Then the influence of point $j\in V$ on point $i\in V$ is defined as
\begin{align*}
I(j,i)=\max_{(X,Y)\in S_j}\dtv(D_i(\cdot\mid X_{-i}),D_i(\cdot\mid Y_{-i})) \;,
\end{align*}
where $D_i(\cdot\mid X_{-i}),D_i(\cdot\mid Y_{-i})$ 
denote the marginal distribution of point $i$ conditioning on $X_{-i}$ and $Y_{-i}$ respectively.
\end{definition}

\begin{definition}[Dobrushin's uniqueness condition]
Let $D$ be a distribution over some set of points $V$.
Then $D$ is said to satisfy Dobrushin's uniqueness condition if $\max_{i\in V}\sum_{j\in V}{I(j,i)}<1$.
\end{definition}

For Ising models,~\cite{chatterjee2005concentration} proves that 
$\max_{i\in V}\sum_{j\ne i}{|\theta_{ij}|}<1$ implies the Dobrushin's uniqueness condition.

\subsection{Basic Facts about Sub-exponential Distributions} \label{ssec:subexp}
The following result establishes that, for any sub-exponential distribution, 
the empirical mean and empirical covariance converge fast 
to the true mean and covariance.

\begin{lemma}[see, e.g.,~\cite{vershynin2018high,kuchibhotla2018moving}]\label{lem:inference}
Let $D$ be a sub-exponential distribution over $\mathbb{R}^d$ 
with mean $\mu$ and covariance $\Sigma$. 
Let $X_1,\ldots,X_n$ be i.i.d.\ samples drawn from $D$, 
$\wh{\mu}_n=\frac{1}{n}\sum_{i=1}^{n}{X_i}$ be the empirical mean, 
and $\wh{\Sigma}_n=\frac{1}{n}\sum_{i=1}^{n}{(X_i-\wh{\mu}_n)(X_i-\wh{\mu}_n)^T}$ 
be the empirical covariance. 
Then there exist constants $c_1,c_2>0$ such that the following holds:
\begin{enumerate}
\item With probability at least $1-2\exp(-t^2)$, we have that
\begin{align*}
\left\|\wh{\mu}_n-\mu\right\|_2\le c_1\max(\delta,\delta^2) \;,
\end{align*}
where $\delta=\sqrt{\frac{d}{n}}+\frac{t}{\sqrt{n}}$, 
and 
\item With probability at least $1-6\exp(-t)$, we have that
\begin{align*}
\ltwo{\wh{\Sigma}_n-\Sigma}\le c_2d\left(\sqrt{\frac{t+\log d}{n}}+\frac{((t+\log d)\log n)^2}{n}\right) \;,
\end{align*}
\end{enumerate}
\end{lemma}

\subsection{Basic Facts on Optimization of Smooth and Strongly Convex Functions} \label{ssec:opt}

In this section, we provide some background on smooth and strongly convex optimization.

\begin{definition} \label{def:ssc}
Let $\Omega \subseteq \R^d$ be a convex set and $f:\Omega \to\R$ be twice continuously differentiable. 
For $m>0$, we say that $f$ is $m$-strongly convex over $\Omega$ if $\nabla^2f(x)\succeq mI$, 
for all $x\in\Omega$. We say that $f$ is $L$-smooth over $\Omega$ if $-LI\preceq\nabla^2f(x)\preceq LI$ 
for all $x\in\Omega$.
\end{definition}

\paragraph{Notation} 
Let $\mathcal{X}\subseteq\mathbb{R}^d$ be a convex set. 
We denote $\Diam(\mathcal{X})$ to be the diameter of $\mathcal{X}$ in Euclidean norm, i.e., $\Diam(\mathcal{X})=\sup_{x,y\in\mathcal{X}}\|x-y\|_2$. 
For an arbitrary point $x\in\mathbb{R}^d$, we denote $P_\mathcal{X}(x)$ to be the Euclidean projection of $x$ 
to $\mathcal{X}$, i.e., $P_\mathcal{X}(x)=\arg\min_{z\in\mathcal{X}}\|z-x\|_2$. 

\medskip

The following projected gradient descent method for minimizing a smooth and strongly convex 
function is standard.

\begin{algorithm}
\SetKwInOut{Input}{Input}
\SetKwInOut{Output}{Output}
\Input{an $m$-strongly convex and $L$-smooth function $f$ over a convex set $\Omega$ and a constant $\delta>0$.}
\Output{an $\wh{x}\in\Omega$ such that $\|\wh{x}-x^*\|_2\le\delta$, where $x^*=\arg\min_{x\in\Omega}f(x)$.}
Let $x^0\in\Omega$ be an arbitrary initial point and $T=O\left(\frac{L}{m}\log\left(\frac{\Diam(\Omega)}{\delta}\right)\right)$.\\
\For{$t$ = $0$ to $T-1$}{
$r^t= x^t-\frac{1}{L}\nabla f(x^t)$.\\
$x^{t+1}=\arg\min_{x\in\Omega}\|x-r^t\|_2$.\\
}
\Return{$x^T$.}
\caption{Projected gradient descent for strongly convex smooth optimization}\label{PGD}
\end{algorithm}

The following fact is standard.

\begin{fact}[see, e.g.,~\cite{nesterov2018lectures}]\label{descent-lem}
Let $f:\Omega\to\mathbb{R}$ be  $L$-smooth and $m$-strongly convex. 
Let $x^*=\arg\min_{x\in\Omega}{f(x)}$. The iterates in Algorithm~\ref{PGD} satisfy
\begin{align*}
\|x^{t+1}-x^*\|_2^2\le\left(1-\frac{m}{L}\right)\|x^t-x^*\|_2^2 \;.
\end{align*}
Therefore, after $T=O\left(\frac{L}{m}\log\left(\frac{\Diam(\Omega)}{\delta}\right)\right)$ iterations, 
we have that $\|x^T-x^*\|_2\le\delta$.
\end{fact}

\section{Basic Properties of Exponential Families}\label{app:exp-basics}

\subsection{Proof of Fact~\ref{fact:exp-family-p1}}
Let $Z(\theta)=\exp(A(\theta))=\sum_{x\in\mathcal{X}}\exp\left(\langle T(x),\theta\rangle\right)$. 
From elementary calculation, we have that
\begin{align*}
\nabla A(\theta)= \nabla\ln Z(\theta)= \frac{\nabla Z(\theta)}{Z(\theta)} = 
\frac{\sum_{x\in\mathcal{X}}\exp(\langle T(x),\theta\rangle)T(x)}{Z(\theta)}=\mu_T \;,
\end{align*}
and
\begin{align*}
\nabla^2A(\theta)
& = \frac{\partial\mu_T}{\partial\theta}
= \frac{\partial}{\partial\theta}\left(\frac{\sum_{x\in\mathcal{X}}\exp(\langle T(x),\theta\rangle)T(x)}{Z(\theta)}\right) 
= \frac{\sum_{x\in\mathcal{X}}\exp(\langle T(x),\theta\rangle)T(x)T(x)^T}{Z(\theta)}\\
&\quad-\frac{\left(\sum_{x\in\mathcal{X}}\exp(\langle T(x),\theta\rangle)T(x)\right)\left(\sum_{x\in\mathcal{X}}\exp(\langle T(x),\theta\rangle)T(x)^T\right)}{Z(\theta)^2}\\
&=\E[T(X)T(X)^T]-\mu_T\mu_T^T=\Sigma_T \;.
\end{align*}

\subsection{Proof of Fact~\ref{fact:exp-family-p2}}
From the definition of KL-divergence, we have that
\begin{align*}
\dkl(P_\theta,P_{\theta'})&=\E_{X\sim P_\theta}\left[\ln\left(\frac{P_\theta(x)}{P_{\theta'}(x)}\right)\right]\\
&=\E_{X\sim P_\theta}[\langle T(X),\theta-\theta'\rangle]-A(\theta)+A(\theta')\\
&=\langle\theta-\theta',\mu_T\rangle-A(\theta)+A(\theta') \;.
\end{align*}

\subsection{Proof of Lemma~\ref{lem:para-to-dis}}\label{app:para-to-dis}
Let $\theta=\wh{\theta}-\theta^*$. 
Define $g(x)=\langle T(x),\theta\rangle-\E_{X\sim P_{\theta^*}}[\langle T(x),\theta\rangle]$.
By definition, we have that $\E_{X\sim P_{\theta^*}}[g(X)]=0$, and for any $x\in\mathcal{X}$,
\begin{align*}
\frac{P_{\wh{\theta}}(x)}{P_{\theta^*}(x)}
&=\frac{\exp(\langle T(x),\wh{\theta}\rangle-A(\wh{\theta}))}{\exp\left(\langle T(x),\theta^*\rangle-A(\theta^*)\right)}
=\frac{\exp(\langle T(x),\theta\rangle)}{\exp(A(\wh{\theta})-A(\theta^*))}
=\frac{\exp(\langle T(x),\theta\rangle)}{\sum_{x\in\mathcal{X}}\exp(\langle T(x),\wh{\theta}\rangle-A(\theta^*))}\\
&=\frac{\exp(\langle T(x),\theta\rangle)}{\sum_{x\in\mathcal{X}}\exp(\langle T(x),\theta\rangle)\cdot\exp(\langle T(x),\theta^*\rangle-A(\theta^*))}
=\frac{\exp(\langle T(x),\theta\rangle)}{\E_{X\sim P_{\theta^*}}[\exp(\langle T(X),\theta\rangle)]}\\
&=\frac{\exp(g(x))}{\E_{X\sim P_{\theta^*}}[\exp(g(X))]}=\exp(g(x))/w \;,
\end{align*}
where $w=\E_{X\sim P_{\theta^*}}[\exp(g(X))]$. 
\new{In order to bound the total variation distance, we bound the $\chi^2$-distance between $P_{\wh{\theta}}$ and $P_{\theta^*}$. Recall that for any two distributions $p,q$ over $\mathcal{X}$, $\chi^2(p,q)\eqdef\int_{\mathcal{X}}\left(\frac{dp}{dq}-1\right)^2dq$, we have that}
\begin{align*}
\chi^2(P_{\wh{\theta}},P_{\theta^*})
&=\E_{X\sim P_{\theta^*}}\left[\left(\frac{P_{\wh{\theta}}(X)}{P_{\theta^*}(X)}-1\right)^2\right] 
=\frac{\E_{X\sim P_{\theta^*}}\left[\left(\exp(g(X))-w\right)^2\right]}{w^2}\\
&=\frac{\E_{X\sim P_{\theta^*}}[\exp(2g(X))]}{w^2}-1\le\E_{X\sim P_{\theta^*}}[\exp(2g(X))]-1 \;,
\end{align*}
where we apply $w\ge1$ in the last inequality, 
since $$w=\E_{X\sim P_{\theta^1}}[\exp(g(X))]\ge\exp\left(\E_{X\sim P_{\theta^1}}[g(X)]\right)=1$$ by Jensen's inequality.
By our assumption, there is a constant $c_1>0$ such that
\begin{align*}
\Pr_{X\sim P_{\theta^*}}[|g(X)-\E[g(X)]|>t]\le2\exp\left(-\frac{c_1t}{\|\theta\|_2}\right) \;.
\end{align*}
Hence, from Fact~\ref{fact:sub-exponential-property}, there is a constant $c_2>0$ such that as long as 
$|\lambda|\le\frac{c_1}{c_2\|\theta\|_2}$, we will have that
$\E_{X\sim P_{\theta^1}}[\exp(\lambda g(X))]\le\exp(c_2^2\lambda^2\|\theta\|_2^2/c_1^2)$. 
Now we assume that $\|\theta\|_2^2=\|\wh{\theta}-\theta^*\|_2^2\le\delta^2\le c_1^2/4c_2^2$ and derive that
\begin{align*}
\chi^2(P_{\wh{\theta}},P_{\theta^*})\le \E_{X\sim P_{\theta^*}}[\exp(2g(X))]-1 
\le\exp(4c_2^2\|\theta\|_2^2/c_1^2)-1 
\le 8c_2^2\|\theta\|_2^2/c_1^2 \;,
\end{align*}
where we apply the elementary inequality $e^x\le1+2x$, for $x\le1$. Therefore,
\begin{align*}
\dtv(P_{\wh{\theta}},P_{\theta^*})\le\sqrt{\frac{\chi^2(P_{\wh{\theta}},P_{\theta^*})}{2}}\le2c_2\|\wh{\theta}-\theta^*\|_2/c_1 \;.
\end{align*}

\subsection{Proof of Lemma~\ref{lem:parameter-to-mean}}\label{app:para-to-mean}
From the definition of $L(\theta,\mu_T)=\langle\theta,\mu_T\rangle-A(\theta)$ 
and Fact~\ref{fact:exp-family-p1}, 
it follows that for any $\theta\in\Omega$, we have that
$$\nabla^2_\theta L(\theta,\mu_T)=-\nabla^2_\theta A(\theta)=-\cov_{X\sim P_\theta}[T(X)] \preceq -c \, I \;.$$ 
Hence, for any fixed $\mu_T \in \R^d$, the objective function $L(\theta,\mu_T)$ is $c$-strongly concave 
and \new{therefore} has a unique maximizer $\theta_{\mu_T} \in \Omega$. 
From Fact~\ref{fact:exp-family-p2}, we have that
\begin{align*}
L(\theta^*,\mu^*_T)-L(\theta',\mu^*_T)& =\langle\theta^*-\theta',\mu_T^*\rangle-A(\theta^*)+A(\theta')=\dkl(P_{\theta^*},P_{\theta'}), \textrm{ and }\\
0 \le L(\theta',\mu'_T)-L(\theta^*,\mu'_T)& =\langle\theta'-\theta^*,\mu'_T\rangle-A(\theta')+A(\theta^*) \;,
\end{align*}
where we used the fact that, given $\mu'_T \in\R^d$, 
$L(\theta,\mu'_T)$ attains its maximum at $\theta=\theta'$ over $\Omega$. 
Adding the above two equations together, we get
\begin{align}\label{KL-upper-bound}
\langle\theta^*-\theta',\mu_T^*-\mu'_T\rangle=\left(L(\theta^*,\mu_T^*)-L(\theta',\mu_T^*)\right)+
\left(L(\theta',\mu'_T)-L(\theta^*,\mu'_T)\right)\ge\dkl(P_{\theta^*},P_{\theta'}) \;.
\end{align}
In addition, from Taylor's theorem, we can rewrite $\dkl(P_{\theta^*}, P_{\theta'})$ as follows
\begin{align}\label{KL-lower-bound}
\dkl(P_{\theta^*},P_{\theta'})
&=\dkl(P_{\theta^*},P_{\theta'})-\dkl(P_{\theta^*},P_{\theta^*})\notag\\
&=\nabla_{\theta'}\dkl(P_{\theta^*},P_{\theta'})\Big|_{\theta'=\theta^*}+
\frac{1}{2}(\theta'-\theta^*)^T\nabla^2_{\theta'}\dkl(P_{\theta^*},P_{\theta'})\Big|_{\theta'=\theta''}(\theta'-\theta^*)\notag\\
&=\frac{1}{2}(\theta'-\theta^*)^T\cov_{X\sim P_{\theta''}}[T(x)](\theta'-\theta^*)\notag\\
&\ge\frac{c}{2}\|\theta'-\theta^*\|_2^2 \;,
\end{align}
where \new{$\theta''=\lambda\theta'+(1-\lambda)\theta^*$ for some $0\le\lambda\le1$ and we apply Fact~\ref{fact:exp-family-p2} in the third equality}.
Combining~\eqref{KL-upper-bound} and~\eqref{KL-lower-bound}, we obtain that
\begin{align*}
\|\theta^*-\theta'\|_2 \, \|\mu^*_T-\mu'_T\|_2 \ge 
\langle\theta^*-\theta',\mu_T^*-\mu'_T\rangle 
\ge \dkl(P_{\theta^*},P_{\theta'})\ge\frac{c}{2}\|\theta'-\theta^*\|_2^2 \;,
\end{align*}
which implies that
\begin{align*}
\|\theta^*-\theta'\|_2 \le \frac{2}{c} \|\mu_T^*-\mu'_T\|_2 \le \frac{2\delta}{c} \;.
\end{align*}
This completes the proof.

\section{Basic Properties of Ising Models}\label{app:ising-basics}

\subsection{Proof of Fact~\ref{fact:cond}}

Let $x_I,x'_I\in\{\pm1\}^I$. We calculate the ratio of conditional probabilities for two configurations $x_I$ and $x'_I$, 
as follows:
\begin{align*}
\frac{\Pr[X_I=x_I\mid X_{-I}=x_{-I}]}{\Pr[X_I=x'_I\mid X_{-I}=x_{-I}]}=\frac{\exp\left(\sum_{i,j\in I}{\theta_{ij}x_ix_j}+\sum_{i\in I}{x_i\left(\theta_i+\sum_{j\notin I}{\theta_{ij}x_j}\right)}\right)}{\exp\left(\sum_{i,j\in I}{\theta_{ij}x'_ix'_j}+\sum_{i\in I}{x'_i\left(\theta_i+\sum_{j\notin I}{\theta_{ij}x_j}\right)}\right)}.
\end{align*}
Therefore, the conditional distribution of $X_I$ conditioning on $X_{-I}=x_{-I}$ is an Ising model 
with interaction matrix $(\theta_{ij})_{i,j\in I}$ and external field $\theta'_i=\theta_i+\sum_{j\notin I}{\theta_{ij}x_j}$.

\subsection{Proof of Fact~\ref{fact:marginal}}
By definition of the Ising model, we can write
\begin{align*}
\Pr[X_i=x_i]
&= \sum_{x_{-i}\in\{\pm1\}^{d-1}}{\Pr[X_{-i}=x_{-i}]\cdot \Pr[X_i=x_i\mid X_{-i}=x_{-i}]}\\
&= \sum_{x_{-i}\in\{\pm1\}^{d-1}}{\Pr[X_{-i}=x_{-i}] \cdot 
\frac{\exp\left(\theta_i x_i+x_i \sum_{j\ne i}{\theta_{ij}x_j}\right)}
{\exp\left(\theta_i x_i + x_i \sum_{j\ne i}{\theta_{ij} x_j}\right)+ \exp\left(-\theta_i x_i - x_i \sum_{j\ne i}{\theta_{ij} x_j}\right)}}\\
& = \sum_{x_{-i} \in \{\pm1\}^{d-1}}{\Pr[X_{-i}=x_{-i}] \cdot 
\frac{\exp\left(2\theta_i x_i + 2 x_i \sum_{j\ne i}{\theta_{ij} x_j}\right)}{1+\exp\left(2\theta_i x_i + 2x_i \sum_{j\ne i}{\theta_{ij}x_j}\right)}} \;.
\end{align*}
Since $X$ is an $(M,\alpha)$-bounded Ising model and the function $f(t)=\frac{e^t}{1+e^t}$ is monotonically increasing, we have that
\begin{align}\label{marginal-bound}
\frac{\exp(-2(\alpha+M))}{1+\exp(-2(\alpha+M))}
\le \frac{\exp\left(2\theta_i x_i + 2 x_i \sum_{j\ne i}{\theta_{ij} x_j}\right)}{1+\exp\left(2 \theta_i x_i + 2 x_i \sum_{j\ne i}{\theta_{ij} x_j}\right)}\le \frac{\exp(2(\alpha+M))}{1+\exp(2(\alpha+M))} \;,
\end{align}
which implies that 
$\frac{\exp(-2(\alpha+M))}{1+\exp(-2(\alpha+M))} 
\le \Pr[X_i=x_i] \le \frac{\exp(2(\alpha+M))}{1+\exp(2(\alpha+M))}$. 

Let $p_i=\Pr[X_i=1]$. We directly calculate $\E[X_i]$ and $\var[X_i]$ as follows.
\begin{align*}
\E[X_i]&=\Pr[X_i=1]-\Pr[X_i=-1]=2p_i-1,\\
\var[X_i]&=1-\E[X_i]^2=1-(2p_i-1)^2=4p_i(1-p_i) \;.
\end{align*}
Hence, from inequality~\eqref{marginal-bound}, we have that 
$$\var[X_i]=4 \, \Pr[X_i=1] \, \Pr[X_i=-1] \ge 4 \left(\frac{\exp(-2(\alpha+M))}{1+\exp(-2(\alpha+M))}\right)^2 \;,$$
which completes the proof.

\section{Omitted Proofs from Section~\ref{sec:exp}} \label{app:exp}

\subsection{Proof of Lemma~\ref{lem:mean-to-parameter}} \label{app:mean-to-parameter}
The following simple claim shows that under Condition~\ref{exp-family-cond}, the likelihood function 
of the exponential family is smooth and strongly convex.

\begin{claim}\label{claim:smooth-convex}
Fix $\mu_T \in \R^d$. For any $\theta \in \Omega$, define $L(\theta,\mu_T)=\langle\theta,\mu_T\rangle-A(\theta)$, 
where $A(\theta)$ is the log-partition function for the exponential family $P_{\theta}$ with sufficient statistics $T(x)$. 
If Condition~\ref{exp-family-cond} holds, then $-L(\theta,\mu_T)$ is $L$-smooth and $m$-strongly convex, 
for some constants $L, m>0$ independent of the vector $\mu_T$.
\end{claim}
\begin{proof}
Let $f(\theta)=-L(\theta,\mu_T)$ and we have that $\nabla f(\theta)=\E_{X\sim P_\theta}[T(X)]-\mu_T$ and 
$\nabla^2f(\theta)=\cov_{X\sim P_\theta}[T(X)]$. From the first statement in Condition~\ref{exp-family-cond}, 
we know that $\cov_{X\sim P_\theta}[T(X)] \succeq m \, I$ for some universal constant $m>0$, 
and thus $f(\theta)$ is $m$-strongly convex. In addition, from the second statement in Condition~\ref{exp-family-cond}, 
we know that there exists a constant $c>0$ such that for any parameter $\theta\in\Omega$ 
and any unit vector $v\in\mathbb{S}^{d-1}$, $\Pr_{X\sim P_\theta}[|\langle v,T(X)-\E[T(X)]\rangle|>t]\le2\exp(-ct),\forall t>0$. 
From Fact~\ref{fact:sub-exponential-property}, we have that $\cov_{X\sim P_\theta}[T(X)] \preceq  L \, I$, 
for some universal constant $L>0$ and thus $f(\theta)$ is $L$-smooth.
\end{proof}

Since $-L(\theta,\mu_T)$ is $L$-smooth and $m$-strongly convex, one can apply Projected Gradient Descent (PGD)  
to efficiently compute the maximum likelihood estimator $\arg\max_{\theta\in\Omega}L(\theta,\mu_T)$ 
for any fixed $\mu_T \in \R^d$. \new{A small wrinkle is that, in order to apply vanilla PGD (Algorithm~\ref{PGD}), 
we need access to exact gradients and projections. 
In our setting, this is not possible in general: For general exponential families, it is 
computationally hard to compute $\nabla(-L(\theta,\mu_T))=\E_{X\sim P_\theta}[T(X)]-\mu_T$ exactly.
To address this minor issue, we need to slightly modify Algorithm~\ref{PGD} and its analysis, 
where we use sufficiently accurate approximations to the gradient and the projection.}

\begin{algorithm}
\SetKwInOut{Input}{Input}
\SetKwInOut{Output}{Output}
\Input{$L$-smooth and $m$-strongly convex function $f$ over $\Omega$ and \new{parameters} 
$\delta,\delta_1,\delta_2>0$.}
\Output{$\wh{x}\in\Omega$ such that $\|\wh{x}-x^*\|_2\le\delta+\frac{\delta_2+\delta_1/L}{1-\sqrt{1-m/L}}$, 
where $x^*=\arg\min_{x\in\Omega}f(x)$.}
Let $x^0\in\Omega$ be an arbitrary initial point and $T=O\left(\frac{L}{m}\log\left(\frac{\Diam(\Omega)}{\delta}\right)\right)$.\\
\For{$t$ = $0$ to $T-1$}{
Compute $g^t$ such that $\|g^t-\nabla f(x^t)\|_2\le\delta_1$.\\
$r^t= x^t- \new{\frac{1}{L}} \, g^t$.\\
Compute $x^{t+1}\in\Omega$ such that $\|x^{t+1}-P_\Omega(r^t)\|_2\le\delta_2$, where $P_\Omega(r^t)=\arg\min_{x\in\Omega}\|x-r^t\|_2$.\\
}
\Return{$x^T$\;}
\caption{Projected gradient descent for strongly convex smooth optimization with approximate gradient and projection}\label{APGD}
\end{algorithm}

\new{The following simple claim adapts the analysis of PGD to work with approximate gradients and projections.}

\begin{claim} \label{clm:PGD-approx}
Let $f: \Omega \to \R$ be $L$-smooth and $m$-strongly convex and $x^*=\arg\min_{x\in\Omega}{f(x)}$. 
The iterates in Algorithm~\ref{PGD} satisfy
\begin{align*}
\|x^{t+1}-x^*\|_2\le\delta_2+\delta_1/L+\sqrt{1-m/L}\|x^t-x^*\|_2 \;.
\end{align*}
Therefore, after $T=O\left(\frac{L}{m}\log\left(\frac{\Diam(\Omega)}{\delta}\right)\right)$ iterations, 
we have that $\|x^T-x^*\|_2\le\delta+\frac{\delta_2+\delta_1/L}{1-\sqrt{1-m/L}}$.
\end{claim}
\begin{proof}
From Fact~\ref{descent-lem}, we have that
\begin{align*}
\ltwo{x^{t+1}-x^*}&\le\ltwo{x^{t+1}-P_\Omega(r^t)}+\ltwo{P_\Omega(r^t)-P_\Omega\left(x^t-\frac{1}{L}\nabla f(x^t)\right)}\\
&\quad+\ltwo{P_\Omega\left(x^t-\frac{1}{L}\nabla f(x^t)\right)-x^*}\\&\le\ltwo{x^{t+1}-P_\Omega(r^t)}+\ltwo{r^t-\left(x^t-\frac{1}{L}\nabla f(x^t)\right)}+\ltwo{P_\Omega\left(x^t-\frac{1}{L}\nabla f(x^t)\right)-x^*}\\
&=\ltwo{x^{t+1}-P_\Omega(r^t)}+\frac{1}{L}\ltwo{g^t-\nabla f(x^t)}+\ltwo{P_\Omega\left(x^t-\frac{1}{L}\nabla f(x^t)\right)-x^*}
\\&\le\delta_2+\delta_1/L+\sqrt{1-m/L}\|x^t-x^*\|_2 \;,
\end{align*}
where we apply $\|P_\Omega(x)-P_\Omega(y)\|_2\le\|x-y\|_2,\forall x, y\in \R^d$ in the second inequality. 
Therefore, we can write
\begin{align*}
\|x^T-x^*\|_2-\frac{\delta_2+\delta_1/L}{1-\sqrt{1-m/L}} 
&\le \sqrt{1-m/L}\left(\|x^{T-1}-x^*\|_2-\frac{\delta_2+\delta_1/L}{1-\sqrt{1-m/L}}\right)\\
&\le (1-m/L)^{T/2}\left(\|x^0-x^*\|_2-\frac{\delta_2+\delta_1/L}{1-\sqrt{1-m/L}}\right)\\
&\le(1-m/L)^{T/2}\Diam(\Omega) \;.
\end{align*}
\end{proof}

Claim~\ref{clm:PGD-approx} tells us that if we are able to efficiently approximate 
the projection of an arbitrary point in $\R^d$ to $\Omega$ and the gradient of the function, 
then we can efficiently solve the underlying minimization problem. 
From Condition~\ref{exp-family-cond}, we can efficiently approximate 
the projection of any point in $\R^d$ within error $1/\poly(d)$. 
Note that for any fixed $\mu_T \in \R^d$, the gradient of the negative likelihood is equal to
$\nabla(-L(\theta,\mu_T))=\E_{X\sim P_{\theta}}[T(X)]-\mu_T$. Therefore, it suffices to show that 
for any given parameter $\theta$, we can efficiently estimate the mean $\E_{X \sim P_{\theta}}[T(X)]$
within small error. This is done in the following claim:

\begin{claim}\label{approx-inference}
Let $P,Q$ be distributions on $\R^d$. Assume that $Q$ is sub-exponential and that $\dtv(P,Q)\le\gamma$ for some parameter $\gamma>0$. Let $\mu$ and $\Sigma$ denote the mean and covariance of distribution $Q$ respectively. Let $X_1,\ldots,X_n$ be i.i.d.\ samples drawn from $P$ and 
$\wh{\mu}_n=\frac{1}{n}\sum_{i=1}^{n}{X_i}$, 
$\wh{\Sigma}_n=\frac{1}{n}\sum_{i=1}^{n}{(X_i-\wh{\mu}_n)(X_i-\wh{\mu}_n)^T}$ be the empirical mean and covariance. Then there exist constants $c_1,c_2>0$ such that the following holds:
\begin{enumerate}
\item With probability at least $1-2\exp(-t^2)-n\gamma$, we have that
$\left\|\wh{\mu}_n-\mu\right\|_2\le c_1\max(\delta,\delta^2)$, 
where $\delta=\sqrt{\frac{d}{n}}+\frac{t}{\sqrt{n}}$.

\item With probability at least $1-6\exp(-t)-n\gamma$, we have that
$$\left\|\wh{\Sigma}_n-\Sigma\right\|_2\le c_2d\left(\sqrt{\frac{t+\log d}{n}}+\frac{((t+\log d)\log n)^2}{n}\right).$$
\end{enumerate}
\end{claim}
\begin{proof}
Let $Y_1,\ldots,Y_n$ be $n$ i.i.d.\ samples drawn from $Q$. Let
\begin{align*}
\mu_n=\frac{1}{n}\sum_{i=1}^{n}{Y_i}, \qquad\text{and}\qquad 
\Sigma_n=\frac{1}{n}\sum_{i=1}^{n}{(Y_i-\mu_n)(Y_i-\mu_n)^T}.
\end{align*} 
By the data processing inequality for the total variation distance, we can write 
$$\dtv(\mu_n,\wh{\mu}_n) \le \dtv((X_1,\ldots,X_n),(Y_1,\ldots,Y_n))\le n\dtv(P,Q) \le n\gamma \;$$ 
and similarly 
$$\dtv(\Sigma_n,\wh{\Sigma}_n) \le \dtv((X_1,\ldots,X_n),(Y_1,\ldots,Y_n)) \le n \dtv(P,Q) \le n\gamma \;.$$ 
We pick optimal couplings $(\wh{\mu}_n,\mu_n)$ and $(\wh{\Sigma}_n,\Sigma_n)$. 
From Lemma~\ref{lem:inference}, there exist constants $c_1,c_2>0$ such that
\begin{align*}
\Pr\left[\left\|\wh{\mu}_n-\mu\right\|_2>c_1\max(\delta,\delta^2)\right]&\le\Pr\left[\wh{\mu}_n\ne\mu_n\right]+\Pr\left[\left\|\mu_n-\mu\right\|_2>c_1\max(\delta,\delta^2)\right]\\
&\le2\exp(-t^2)+n\gamma \;,
\end{align*}
and
\begin{align*}
&\quad\Pr\left[\left\|\wh{\Sigma}_n-\Sigma\right\|_2>c_2d\left(\sqrt{\frac{t+\log d}{n}}+\frac{((t+\log d)\log n)^2}{n}\right)\right]\\&\le\Pr\left[\wh{\Sigma}_n\ne\Sigma_n\right]+\Pr\left[\left\|\Sigma_n-\Sigma\right\|_2>c_2d\left(\sqrt{\frac{t+\log d}{n}}+\frac{((t+\log d)\log n)^2}{n}\right)\right]\\&\le6\exp(-t)+n\gamma \;.
\end{align*}
\end{proof}

We are now ready to prove Lemma~\ref{lem:mean-to-parameter}.

\begin{proof}[Proof of Lemma~\ref{lem:mean-to-parameter}]
Let $L(\theta,\mu'_T)=\langle\theta,\mu'_T\rangle-A(\theta),\forall\theta\in\Omega$ and $\theta'=\arg\max_{\theta\in\Omega}L(\theta,\mu_T')$. By Lemma~\ref{lem:parameter-to-mean},
we have that $\|\theta'-\theta^*\|_2\le O(\delta)$. 
If we pick $\delta_1=\delta_2=\delta$ and apply Algorithm~\ref{APGD} to the function $-L(\theta,\mu_T')$, 
it will return a point $\wh{\theta} \in \Omega$ with $\|\wh{\theta}-\theta'\|_2\le O(\delta)$, 
since \new{by Claim~\ref{claim:smooth-convex}} $-L(\theta,\mu_T')$ is $L$-smooth and $m$-strongly convex, for some universal constants $L,m>0$. 
This implies that $\|\wh{\theta}-\theta^*\|_2 \le \|\wh{\theta}-\theta'\|_2+\|\theta'-\theta^*\|_2 \le O(\delta)$.

Now we show that the above process is efficient and bound the failure probability. By Condition~\ref{exp-family-cond}, $\Diam(\Omega)\le\exp(d^c)$ for some constant $c>0$. 
Given an arbitrary $\theta\in\Omega$, we can sample from a distribution within total variation distance $\gamma=\frac{\delta^2\zeta}{2d(d^2+\log(1/\delta))\log\left(\frac{4(d^c+\log(1/\delta))}{\zeta}\right)}$ from $P_{\theta}$ in time $\poly\left(\frac{d}{\delta\zeta}\right)$. 
Hence, if we pick $t=\sqrt{\log\left(\frac{4(d^c+\log(1/\delta))}{\zeta}\right)}$ and $n=\Omega(t^2d/\delta^2)$ in Claim~\ref{approx-inference}, 
we are able to estimate the gradient $\nabla_\theta(-L(\theta,\mu'_T))=\E_{X\sim P_\theta}[T(X)]-\mu'_T$ within error $\delta$ with probability at least $1-O\left(\frac{\zeta}{d^c+\log(1/\delta)}\right)$. Since there are $T=O\left(\frac{L}{m}\log\left(\frac{\Diam(\Omega)}{\delta}\right)\right)=O(d^c+\log(1/\delta))$ iterations, by union bound, the algorithm will output a $\wh{\theta}$ with $\|\wh{\theta}-\theta^*\|_2\le O(\delta)$ with probability at least $1-\zeta$.
\end{proof}

\subsection{Proof of Proposition~\ref{prop:exponential-family-p3}} \label{app:prop:exponential-family-p3}

Let $Z(\theta)=\exp(A(\theta))$ be the normalizing factor. 
Fix $i,j,k\in[d]$. We calculate the partial derivative $\frac{\partial(\Sigma_T)_{ij}}{\partial\theta_k}$ as follows.
\begin{align*}
\frac{\partial(\Sigma_T)_{ij}}{\partial\theta_k}&=\frac{\partial}{\partial\theta_k}\left(\E_{X\sim P_\theta}[(T(X)-\mu_T)_i(T(X)-\mu_T)_j]\right)\\&=\frac{\partial}{\partial\theta_k}\left(\frac{\sum_{x\in\mathcal{X}}\exp(\langle T(x),\theta\rangle)(T(x)-\mu_T)_i(T(x)-\mu_T)_j}{Z(\theta)}\right)\\&=\sum_{x\in\mathcal{X}}\frac{\partial}{\partial\theta_k}\left(\frac{\exp(\langle T(x),\theta\rangle)(T(x)-\mu_T)_i(T(x)-\mu_T)_j}{Z(\theta)}\right)\\&=\sum_{x\in\mathcal{X}}\frac{\frac{\partial}{\partial\theta_k}\left(\exp(\langle T(x),\theta\rangle)(T(x)-\mu_T)_i(T(x)-\mu_T)_j\right)}{Z(\theta)}\\&\quad-\frac{\partial Z(\theta)}{\partial\theta_k}\sum_{x\in\mathcal{X}}{\frac{\exp(\langle T(x),\theta\rangle)(T(x)-\mu_T)_i(T(x)-\mu_T)_j}{Z(\theta)^2}} \;.
\end{align*}
Noting that
\begin{align*}
&\quad\frac{\partial}{\partial\theta_k}\left(\exp(\langle T(x),\theta\rangle)(T(x)-\mu_T)_i(T(x)-\mu_T)_j\right)\\&=\new{\exp(\langle T(x),\theta\rangle)}T(x)_k(T(x)-\mu_T)_i(T(x)-\mu_T)_j-\frac{\partial(\mu_T)_i}{\partial\theta_k}\exp(\langle T(x),\theta\rangle)(T(x)-\mu_T)_j\\
&\quad-\frac{\partial(\mu_T)_j}{\partial\theta_k}\exp(\langle T(x),\theta\rangle)(T(x)-\mu_T)_i\\
&=\new{\exp(\langle T(x),\theta\rangle)}T(x)_k(T(x)-\mu_T)_i(T(x)-\mu_T)_j-(\Sigma_T)_{ik}\exp(\langle T(x),\theta\rangle)(T(x)-\mu_T)_j\\
&\quad-(\Sigma_T)_{jk}\exp(\langle T(x),\theta\rangle)(T(x)-\mu_T)_i \;,
\end{align*}
we have that
\begin{align*}
&\quad\sum_{x\in\mathcal{X}}\frac{\frac{\partial}{\partial\theta_k}\left(\exp(\langle T(x),\theta\rangle)(T(x)-\mu_T)_i(T(x)-\mu_T)_j\right)}{Z(\theta)}\\
&=\sum_{x\in\mathcal{X}}\frac{\new{\exp(\langle T(x),\theta\rangle)}T(x)_k(T(x)-\mu_T)_i(T(x)-\mu_T)_j}{Z(\theta)}-\frac{(\Sigma_T)_{ik}\sum_{x\in\mathcal{X}}\exp(\langle T(x),\theta\rangle)(T(x)-\mu_T)_j}{Z(\theta)}\\
&\quad-\frac{(\Sigma_T)_{jk}\sum_{x\in\mathcal{X}}\exp(\langle T(x),\theta\rangle)(T(x)-\mu_T)_i}{Z(\theta)}\\&=\E_{X\sim P_\theta}[T(x)_k(T(x)-\mu_T)_i(T(x)-\mu_T)_j]-(\Sigma_T)_{ik}\E_{X\sim P_\theta}[(T(x)-\mu_T)_j]\\
&\quad-(\Sigma_T)_{jk}\E_{X\sim P_\theta}[(T(x)-\mu_T)_i]\\&=\E_{X\sim P_\theta}[T(x)_k(T(x)-\mu_T)_i(T(x)-\mu_T)_j] \;.
\end{align*}
Therefore,
\begin{align*}
\frac{\partial(\Sigma_T)_{ij}}{\partial\theta_k}&=\sum_{x\in\mathcal{X}}\frac{\frac{\partial}{\partial\theta_k}\left(\exp(\langle T(x),\theta\rangle)(T(x)-\mu_T)_i(T(x)-\mu_T)_j\right)}{Z(\theta)}\\&\quad-\frac{\partial Z(\theta)}{\partial\theta_k}\sum_{x\in\mathcal{X}}{\frac{\exp(\langle T(x),\theta\rangle)(T(x)-\mu_T)_i(T(x)-\mu_T)_j}{Z(\theta)^2}}\\&=\E_{X\sim P_\theta}[T(x)_k(T(x)-\mu_T)_i(T(x)-\mu_T)_j]\\&\quad-\left(\frac{\sum_{x\in\mathcal{X}}\exp(\langle T(x),\theta\rangle)T(x)_k}{Z(\theta)}\right)\left(\sum_{x\in\mathcal{X}}{\frac{\exp(\langle T(x),\theta\rangle)(T(x)-\mu_T)_i(T(x)-\mu_T)_j}{Z(\theta)}}\right)\\&=\E_{X\sim P_\theta}[T(x)_k(T(x)-\mu_T)_i(T(x)-\mu_T)_j]-\E_{X\sim P_\theta}[T(x)_k]\E_{X\sim P_\theta}[(T(x)-\mu_T)_i(T(x)-\mu_T)_j]\\&=\E_{X\sim P_\theta}[(T(x)-\mu_T)_i(T(x)-\mu_T)_j(T(x)-\mu_T)_k] \;.
\end{align*}
This completes the proof.

\subsection{Proof of Lemma~\ref{third-moment}} \label{app:third-moment}

Let $\theta\in\Omega$ and $P_\theta$ be the corresponding exponential family with sufficient statistics $T(x)$. Let $\mu_T(\theta)=\E_{X\sim P_\theta}[T(x)]$ and $\Sigma_T(\theta)=\cov_{X\sim P_\theta}[T(x)]$. Let $v\in\mathbb{S}^{d-1}$ be a unit vector such that $\|\Sigma_T(\theta^1)-\Sigma_T(\theta^2)\|_2=\left|v^T(\Sigma_T(\theta^1)-\Sigma_T(\theta^2))v\right|$. Define $f(\theta)=v^T\Sigma_T(\theta)v$. 
By the mean value theorem, we have that
\begin{align*}
\|\Sigma_T(\theta^1)-\Sigma_T(\theta^2)\|_2&=\left|v^T\Sigma_T(\theta^1)v-v^T\Sigma_T(\theta^2)v\right|=|f(\theta^1)-f(\theta^2)|\\&=|\langle\nabla f(\wt{\theta}),\theta^1-\theta^2\rangle|\\&\le\ltwo{\nabla f(\wt{\theta})}\cdot\|\theta^1-\theta^2\|_2 \;,
\end{align*}
where $\wt{\theta}=\lambda\theta^1+(1-\lambda)\theta^2$ for some $0\le\lambda\le1$. Therefore, we only need to show that $\ltwo{\nabla f(\wt{\theta})}$ is upper bounded by a universal constant $c'>0$. Let $w\in\mathbb{S}^{d-1}$ be the unit vector such that $\ltwo{\nabla f(\wt{\theta})}=\langle w,\nabla f(\wt{\theta})\rangle$. By our definition of function $f(\theta)$, we have that
\begin{align*}
\ltwo{\nabla f(\wt{\theta})}&=\langle w,\nabla f(\wt{\theta})\rangle=\sum_{k=1}^{d}{\frac{\partial f(\wt{\theta})}{\partial\theta_k}\cdot w_k}=\sum_{k=1}^{d}{v^T\left(\frac{\partial\Sigma_T(\wt{\theta})}{\partial\theta_k}\right)v\cdot w_k}=\sum_{i,j,k\in[d]}{v_iv_jw_k\frac{\partial(\Sigma_T)_{ij}(\wt{\theta})}{\partial\theta_k}}\\&=\sum_{i,j,k\in[d]}{v_iv_jw_k\E_{X\sim P_{\wt{\theta}}}\left[(T(X)-\mu_T(\wt{\theta}))_i(T(X)-\mu_T(\wt{\theta}))_j(T(X)-\mu_T(\wt{\theta}))_k\right]}\\&=\E_{X\sim P_{\wt{\theta}}}[\langle T(X)-\mu_T(\wt{\theta}),v\rangle^2\langle T(X)-\mu_T(\wt{\theta}),w\rangle]\\&\le\sqrt{\E_{X\sim P_{\wt{\theta}}}[\langle T(X)-\mu_T(\wt{\theta}),v\rangle^4]}\cdot\sqrt{\E_{X\sim P_{\wt{\theta}}}[\langle T(X)-\mu_T(\wt{\theta}),w\rangle^2]},
\end{align*}
where we apply Proposition~\ref{prop:exponential-family-p3} in the fifth equality and the last inequality comes from the Cauchy Schwarz inequality. From Fact~\ref{fact:sub-exponential-property}, we know that both $\E_{X\sim P_{\wt{\theta}}}[\langle T(X)-\mu_T(\wt{\theta}),\new{v}\rangle^4]$ and $\E_{X\sim P_{\wt{\theta}}}[\langle T(X)-\mu_T(\wt{\theta}),w\rangle^2]$ are upper bounded by universal constants. Hence we obtain that $\ltwo{\nabla f(\wt{\theta})}\le c'$ for some universal constant $c'>0$.

\section{Omitted Proofs from Section~\ref{sec:ising}} \label{app:ising}

\subsection{Proof of Lemma~\ref{lem:conc-Ising}} \label{app:lem:conc-Ising}

\begin{fact}[\cite{gotze2019higher}]\label{fact:f-conc-Ising}
Let $X \sim P_{\theta}$ be an Ising model satisfying Dobrushin's condition and $\max_{i\in[d]}|\theta_i| \le \alpha$, 
where $\alpha>0$ is an absolute constant. Let $f:\{\pm1\}^d \to \R$ be an arbitrary function. 
Define function $Df: \{\pm1\}^d \to \R^d$ as $Df(x)_i=\frac{f(x_{i+})-f(x_{i-})}{2}, \forall x \in \{\pm1\}^d, \forall i \in [d]$, 
where $x_{i+}$ is the vector obtained from $x$ by replacing the $i$-th coordinate with $1$ 
and $x_{i-}$ is the one that is obtained by replacing the $i$-th coordinate with $-1$. 
Define the function $Hf: \{\pm1\}^d \to \R^{d \times d}$ as $Hf(x)_{ij} = D(Df(x)_j)(x)_i, \forall x \in \{\pm1\}^d, \forall i, j\in[d]$. 
If $\E[\|Df(X)\|_2^2] \le 1$ and $\|Hf(x)\|_F^2 \le 1,\forall x \in \{\pm1\}^d$, then there is a constant $c(\alpha,\eta)>0$ such that
\begin{align*}
\Pr[|f(X)-\E[f(X)]|>t] \le 2 \exp(-c(\alpha,\eta) \, t) \;,
\end{align*}
where $\eta>0$ is the constant in Definition~\ref{def:high-temp}.
\end{fact}

\begin{proof}[Proof of Lemma~\ref{lem:conc-Ising}]
Let $A\in \R^{d\times d}$ be a symmetric matrix with zero diagonal and 
$b\in\R^d$ be such that $2\|A\|_F^2+\|b\|_2^2=1$. 
Let $f(X)=(X-v)^TA(X-v)+b^TX$. From Fact~\ref{fact:var-ub}, we can write
\begin{align*}
\E\left[\|Df(X)\|_2^2\right]&=\frac{1}{4}\sum_{i=1}^{d}{\E\left[\left(f(X_{i+})-f(X_{i-})\right)^2\right]}=\sum_{i=1}^{d}{\E\left[\Bigg(b_i+2\sum_{j\ne i}{A_{ij}(X_j-v_j)}\Bigg)^2\right]}\\&=\sum_{i=1}^{d}{\E\Bigg[\Bigg(b_i+2\sum_{j\ne i}A_{ij}(X_j-\E[X_j])+2\sum_{j\ne i}A_{ij}(\E[X_j]-v_j)\Bigg)^2\Bigg]}
\\&\le3\sum_{i=1}^{d}b_i^2+12\sum_{i=1}^{d}\var\Bigg[\sum_{j\ne i}{A_{ij}X_j}\Bigg]+12\sum_{i=1}^{d}\Bigg(\sum_{j\ne i}A_{ij}(\E[X_j]-v_j)\Bigg)^2\\&\le3\sum_{i=1}^{d}b_i^2+12\sum_{i=1}^{d}\var\Bigg[\sum_{j\ne i}{A_{ij}X_j}\Bigg]+12\|A\|_F^2\|\E[X]-v\|_2^2\\&\le3\|b\|_2^2+(c'+12\delta^2)\|A\|_F^2 \;,
\end{align*}
where in the first inequality we used the elementary identity $3(a^2+b^2+c^2) \ge(a+b+c)^2, \forall a,b,c \in\R$, 
and $c'>0$ is an absolute constant. In addition, we have that
\begin{align*}
Hf(x)_{ij}=\frac{Df(x_{i+})_j-Df(x_{i-})_j}{2}=\frac{f(x_{i+,j+})-f(x_{i+,j-})-f(x_{i-,j+})+f(x_{i-,j-})}{4}=A_{ij} \;,
\end{align*}
which implies that $\|Hf\|_F^2=\sum_{i,j\in[d]}A_{ij}^2=\|A\|_F^2$.

Hence, after a renormalization by $1/\sqrt{\max(3,c'+12\delta^2)(\|A\|_F^2+\|b\|_2^2)}$, 
the assumptions in Fact~\ref{fact:f-conc-Ising} are satisfied, and we have that
\begin{align*}
\Pr[|f(X)-\E[f(X)]|>t]\le2\exp\left(-\frac{ct}{(\|A\|_F^2+\|b\|_2^2)^{1/2}}\right) \;,
\end{align*}
where $c>0$ is an absolute constant.
\end{proof}

\end{document}